%% file: acl_latex.tex
\pdfoutput=1
\documentclass[11pt]{article}
\usepackage[final]{acl}
\usepackage{times}
\usepackage{xcolor}
\usepackage{latexsym}
\usepackage[T1]{fontenc}
\usepackage[utf8]{inputenc}
\usepackage{microtype}
\usepackage{inconsolata}
\usepackage{graphicx}
\usepackage{booktabs}
\usepackage{amsfonts}
\usepackage{nicefrac}
\usepackage{microtype}
\usepackage{comment}
\usepackage{subcaption}
\usepackage{caption}
\usepackage{float}
\usepackage{amsthm, amsmath, amsfonts, amssymb}
\usepackage{mathtools}
\usepackage{extarrows}
\usepackage{multirow}
\usepackage{enumitem}
\input{math_commands}

\usepackage{wrapfig}
\usepackage{lipsum}
\theoremstyle{plain}
\newtheorem{theorem}{Theorem}[section]

\newtheorem{lemma}[theorem]{Lemma}

\theoremstyle{definition}
\newtheorem{definition}[theorem]{Definition}
\newtheorem{assumption}[theorem]{Assumption}
\theoremstyle{remark}
\newtheorem{remark}[theorem]{Remark}
\usepackage[normalem]{ulem}
\usepackage{longtable}
\usepackage{makecell}
\usepackage{arydshln}
\usepackage{adjustbox}
\usepackage{threeparttable}
\usepackage{xurl}

\usepackage{tcolorbox}

\usepackage{colortbl}
\definecolor{forward}{RGB}{84, 130, 53}
\definecolor{inverse}{RGB}{47, 85, 151}
\definecolor{resist}{RGB}{128, 0, 128}
\definecolor{rebound}{RGB}{133, 19, 33}

\definecolor{def}{RGB}{119, 228, 200}
\definecolor{thm}{RGB}{69, 53, 193}

\newcommand{\forward}[1]{\textbf{\textcolor{forward}{#1}}}
\newcommand{\inverse}[1]{\textbf{\textcolor{inverse}{#1}}}
\newcommand{\resist}[1]{\textit{#1}}
\newcommand{\rebound}[1]{\textit{#1}}

\newtcolorbox{thmbox}[1][]{colback=thm!5!white,colframe=thm!60!black,boxsep=-4pt,grow to left by=4pt,left=10pt,grow to right by=4pt,right=10pt,top=10pt,bottom=10pt,#1}
\newtcolorbox{defbox}[1][]{colback=def!5!white,colframe=def!60!black,boxsep=-4pt,grow to left by=4pt,left=10pt,grow to right by=4pt,right=10pt,top=10pt,bottom=10pt,#1}

\title{Language Models Resist Alignment: Evidence From Data Compression}

\author{{\normalfont Jiaming Ji$^1$\footnotemark[1],} {\normalfont Kaile Wang$^1$\footnotemark[1],} {\normalfont Tianyi Qiu$^1$\footnotemark[1],} {\normalfont Boyuan Chen$^1$\footnotemark[1],} {\normalfont Jiayi Zhou$^1$\footnotemark[1]}\\
{Changye Li$^1$}, {Hantao Lou$^1$},  {Juntao Dai$^{12}$}, {Yunhuai Liu$^3$}, {Yaodong Yang}$^{12\dag}$\\
\vspace{-0.8em} \\
$^1$Institute for Artificial Intelligence, Peking University \\
$^2$Beijing Academy of Artificial Intelligence \\
$^3$School of Computer Science, Peking University 
}

\begin{document}
\maketitle
{
\renewcommand{\thefootnote}{\fnsymbol{footnote}}
\footnotetext[1]{Equal contribution. $^\dag$Corresponding author. If you have any questions, feel free to email \{jiamg.ji, wkl, tianyi.qiu,\\ cbylll, gaiejj\}@stu.pku.edu.cn, yaodong.yang@pku.edu.cn.}
}
\begin{abstract}
Large language models (LLMs) may exhibit unintended or undesirable behaviors.
Recent works have concentrated on aligning LLMs to mitigate harmful outputs. 
Despite these efforts, some anomalies indicate that even a well-conducted alignment process can be easily circumvented, whether intentionally or accidentally. 
Does alignment fine-tuning yield have robust effects on models, or are its impacts merely \textit{superficial}?
In this work, we make the first exploration of this phenomenon from both theoretical and empirical perspectives. Empirically, we demonstrate the \textit{elasticity} of post-alignment models, \textit{i.e.}, the tendency to revert to the behavior distribution formed during the pre-training phase upon further fine-tuning. 
Leveraging compression theory, we formally deduce that fine-tuning disproportionately undermines alignment relative to pre-training, potentially by orders of magnitude.
We validate the presence of \textit{elasticity} through experiments on models of varying types and scales.
Specifically, we find that model performance declines rapidly before reverting to the pre-training distribution, after which the rate of decline drops significantly.
Furthermore, we further reveal that \textit{elasticity} positively correlates with the increased model size and the expansion of pre-training data.
Our findings underscore the need to address the inherent \textit{elasticity} of LLMs to mitigate their resistance to alignment.\footnote{The model weight and code are available at \url{pku-lm-resist-alignment.github.io}.}
\end{abstract}

\begin{figure}[ht]
    \centering
    \includegraphics[width=0.48\textwidth]{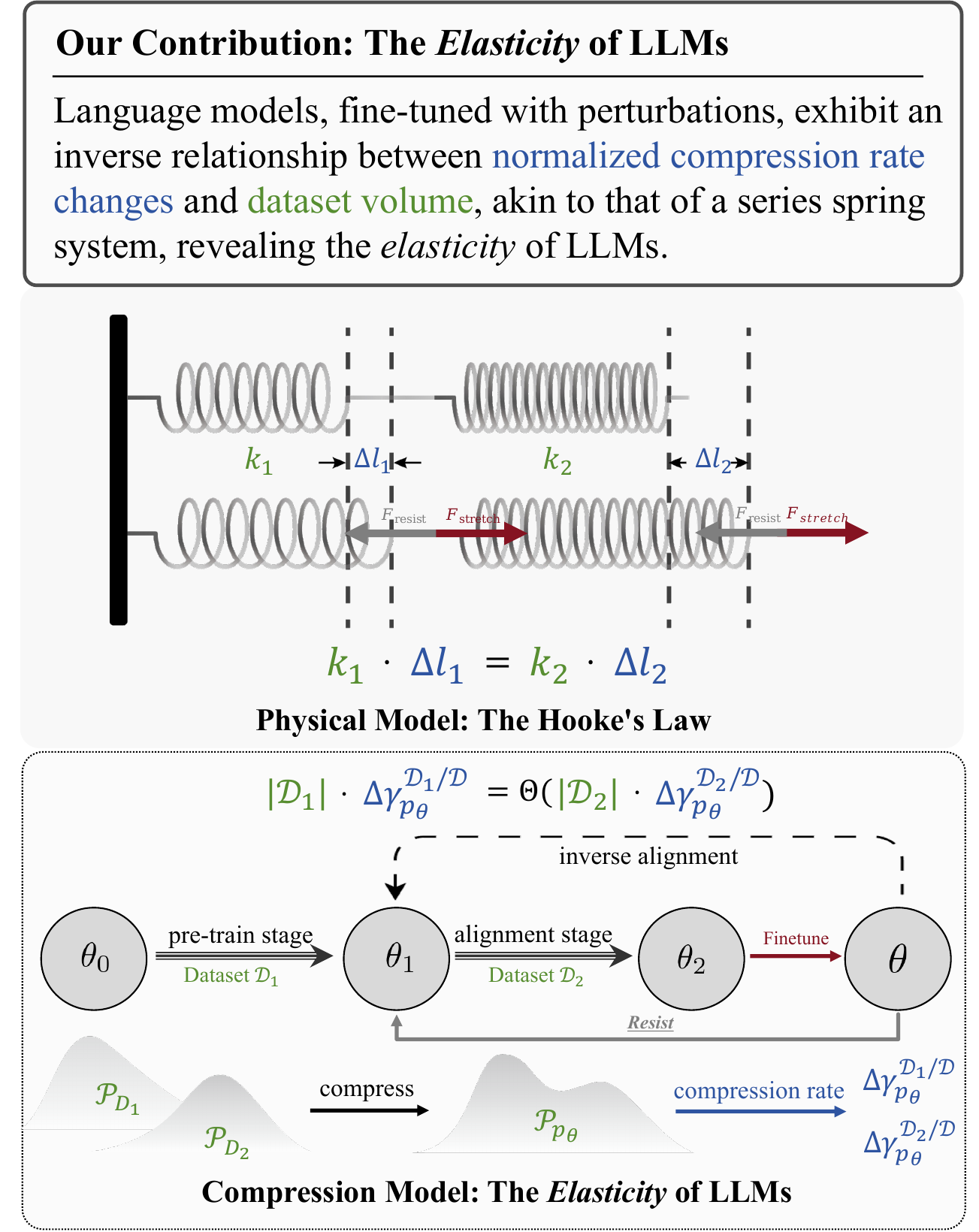}
    \caption{\textbf{The \textit{Elasticity} of Language Models.} The change in normalized compression rates (${\textcolor{inverse}{\Delta\gamma_{p_{\vtheta}}^{\mathcal{D}_i/\mathcal{D}}}}$) and the dataset volume ($\textcolor{forward}{|\mathcal{D}_i|}$) follows an inverse proportionality law after perturbations, which is akin to the relationship between spring deformation ($\textcolor{inverse}{\Delta l_i}$) and stiffness ($\textcolor{forward}{k_i}$) in coupled springs. We conjecture that the \textit{elasticity} causes language models to resist alignment, enabling the possibility of \textit{inverse alignment}. }
    \label{fig: main}
\end{figure}

\section{Introduction}
\label{sec:introduction}

Large language models (LLMs) have shown remarkable capabilities~\cite{achiam2023gpt,zhang2025survey}. However, due to the inevitable biases and harmful content present in training datasets~\citep{bai2022training, ji2024beavertails, qian2024towards, lin2025against}, LLMs often exhibit behaviors that deviate from human intentions, a phenomenon we refer to as \textit{model misalignment}. Training-based alignment methods, including supervised fine-tuning (SFT), reinforcement learning with human feedback (RLHF)~\citep{ouyang2022training}, and other derivatives~\cite{rafailov2024direct, bai2022constitutional, lee2023rlaif, gulcehre2023reinforced, dong2023raft, xiong2024iterative, li2023remax, zhou2023beyond, zhou2025sequence,ji2024aligner}, are the dominant approaches for aligning models. These methods aim to optimize model behavior by rejecting harmful distributions, ensuring LLMs remain consistent with human intentions and values~\cite{ji2023ai, casper2023open}.

However, these alignment methods do not truly penetrate the model representations but merely perform \textit{superficial alignment}~\cite{qi2024safety, cohen2024rl,wen2024language}. Recent studies have shown that highly safety-aligned models can become unsafe again with minimal fine-tuning\cite{yang2023shadow, zhou2024lima}. Furthermore, fine-tuning aligned LLMs on non-malicious datasets may also weaken models' safety mechanisms~\cite{qi2024safety,jain2023mechanistically}. 
\begin{center}
    \textbf{\textit{Why is alignment so fragile?}}
\end{center}

In this work, we make the first exploration of the possible mechanism behind the counterintuitive phenomenon: the existence of an alignment resistance mechanism in language models. This mechanism may limit the alignment process of LLMs to superficial adjustments. It could allow the reversal or revocation of alignment through a series of technical measures, a concept we refer to as \textit{inverse alignment}. What drives language models to resist alignment? How does this mechanism lead to \textit{inverse alignment}? Our key contributions are summarized as follows:

\begin{itemize}[left=0cm]
    \item \textbf{(Phenomenon)} We uncover that language models exhibit \textit{elasticity}, as illustrated in Figure \ref{fig: main} and Theorem \ref{theorem: main}. It encompasses \resist{resistance}: pre-trained models tend to retain their original distribution; and \rebound{rebound}: the deeper alignment of models, the faster they return to the pre-trained distribution under reverse finetuning. Moreover, The model’s change in compression rates $\textcolor{inverse}{\Delta \gamma_{p_\vtheta}^{\mathcal{D}_i/\mathcal{D}}}$ across different datasets is inversely proportional to their sizes $\textcolor{forward}{|\mathcal{D}_i|}$, which is analogous to the deformation behavior of a series of springs, as illustrated in Section \ref{sec: elasticity_mean}. 
    \item \textbf{(Mechanism)} We systematically model the training and alignment process of language models through compression theorem, as detailed in Section \ref{sec: compress_model}. We elaborate on the compression protocol of language models to explore their training and alignment processes, laying a foundation for subsequent research on \textit{elasticity}.
    \item \textbf{(Validation)} We experimentally observe consistent \resist{resistance} and \rebound{rebound} phenomena across various LLMs, as detailed in Section \ref{sec:4}. This highlights the universality of \textit{elasticity} and the need for systematic approaches to achieve robust and deep alignment.
\end{itemize}

\section{Related Work} \label{sec: related_work}
\paragraph{The Fragility of LLMs Alignment} Pre-trained LLMs often generate offensive content. Recent initiatives \cite{ouyang2022training, bai2022training, yang2024alignment} have aimed to align these models to minimize harmful outputs \cite{yang2023rlcd, cui2024ultrafeedback}. However, research indicates that even well-aligned models can be compromised easily, and fine-tuning them on non-malicious datasets might unintentionally impair their safety mechanisms \cite{yang2023shadow, qi2024finetuning, hubinger2024sleeper, dong2024attacks}. Moreover, recent study reveals that models selectively adhere to training objectives during the training phase to preserve their inherent preferences \cite{greenblatt2024alignment, lang2024your, park2024ai}. Why is alignment so fragile? \citet{wei2024assessing} use weight attribution to separate safety-critical and utility-related regions at both neuron and rank levels. \citet{qi2024safety} propose the concept of shallow safety alignment, arguing that safety alignment should reach beyond surface-level tokens to shape the model’s internal mechanisms.

\section{What is \textit{Elasticity}?}
\label{sec:2}
We discover that language models exhibit elasticity, which leads to resistance to alignment. In this section, we formally introduce the definition of \textit{elasticity}, along with the compression theory tools used in the analysis. Firstly, we review the training
alignment objective and the compression theorem.

\subsection{Preliminaries}
\paragraph{Pre-training.} During pre-training, an LLM acquires foundational language comprehension and reasoning abilities by processing vast quantities of unstructured text. The pre-train loss $\mathcal{L}_{\text{PT}}(\vtheta ; \mathcal{D}_{\text{PT}})$ is defined as follows:
\begin{align*}
\label{eq:pretrain}
\mathcal{L}_{\text{PT}}(\vtheta ; \mathcal{D}_{\text{PT}}) = - \mathbb{E}_{(\vx,x_{N}) \sim \mathcal{D}_{\mathrm{PT}}}\left[ \log p_{\vtheta} \left(x_{N} \big| \vx \right) \right],
\end{align*}
where $\vx=(x_0,\cdots,x_{N-1})$ and $N \in \mathbb{N}$, such that $(x_0,\cdots,x_N)$ forms a prefix in some piece of pre-training text. $\mathcal{D}_{\text{PT}}$ stands for pre-training dataset.

\paragraph{Supervised Fine-tuning (SFT).} SFT adjusts the pre-trained models to follow specific instructions, utilizing a smaller dataset compared to the pre-training corpus to ensure model alignment with target tasks. 
For $\mathcal{D}_{\mathrm{SFT}} = \left\{\left( \vx^{i}, \vy^{i} \right) \right\}_{i=1}^{N}$ sampled from a high-quality distribution, SFT aims to minimize the negative log-likelihood loss:
\begin{equation*}
\label{eq:sft}
\mathcal{L}_{\text{SFT}}(\vtheta ; \mathcal{D}_{\text{SFT}}) = - \mathbb{E}_{(\vx,\bf{y) \sim \mathcal{D}_{\mathrm{SFT}}}}\left[ \log p_{\vtheta} \left(\vy \big| \vx \right) \right].
\end{equation*}
Given that $\mathbb{E}_{(\vx,\bf{y) \sim \mathcal{D}_{\text{SFT}}}}\left[  \log{p_{\mathcal{D}} \left(\vy \big| \vx \right)} \right]$ is fixed when specifying $\mathcal{D}_{\text{SFT}}$, the optimization objective $\mathcal{L}_{\text{SFT}}$ becomes the Kullback-Leibler (KL) divergence between $p_{\vtheta}$ and the SFT distribution.

\paragraph{Lossless Compression.}
The goal of lossless compression is to find a compression protocol that encodes a given dataset $\mathcal{D}$ and its distribution $\mathcal{P}_{\mathcal{D}}$ with the smallest possible expected length and allows for a decoding scheme that can perfectly reconstruct the original dataset. According to Shannon's source coding theorem \cite{shannon1948mathematical}, for a random variable follows $\mathcal{P}_{\mathcal{D}}$, the expected code length $\mathcal{L}$ of any lossless compression protocol satisfies:
\begin{equation*}
    \mathcal{L} \geq H\left(\mathcal{P}_{\mathcal{D}}\right),
\end{equation*}
where $H\left(\mathcal{P}_{\mathcal{D}}\right)$ is the Shannon entropy of $\mathcal{P}_{\mathcal{D}}$.

\paragraph{Compression and Prediction.}
Compression and prediction are tightly interconnected. Consider a model $p_{\vtheta}$ and $\vx=\left(x_{0},\cdots,x_{m-1}\right)$ derived from a dataset $\mathcal{D}$, the expected code length $\mathcal{L}$ under arithmetic coding \cite{witten1987arithmetic} satisfies:
\begin{equation*}
   \mathcal{L} = \mathbb{E}_{\vx \sim \mathcal{D}} \left[\sum_{0\leq k\leq m}-\log_{2}p_{\vtheta}\left(x_{k}\big|x_{0,\cdots,k-1}\right)\right],
\end{equation*}
which is the current training objective of language models. Minimizing log-likelihood loss is equivalent to minimizing the compression rate when models act as a lossless compressor. Thus, optimal compression and prediction are equivalent \cite{deletang2023language, hutter2005universal}. Experiments show the equivalence between large language model prediction and compression \cite{deletang2023language}, and that compression performance correlates linearly with intelligence \cite{huang2024compression}.

\subsection{The Compression Protocol of LLMs} \label{sec: compress_model}
We aim to study the dynamic changes in language models during training and alignment. Given the equivalence between language model training and data compression \cite{deletang2023language}, a feasible modeling approach is to treat the language model as a lossless compression protocol. The process of training and aligning the model on different datasets can be equivalently viewed as the joint compression of these datasets by the protocol. The model’s compression rate on various datasets serves as a surrogate metric for the loss during training and alignment. In this section, we detail the modeling specifics of this compression protocol.

Considering the impact of tokenization on compression rate, we use tokenized sequences as input and output modalities. For simplicity, we assume the tokenized vocabulary consists of only binary tokens (specifically \texttt{0/1}).

\begin{defbox}
\begin{definition}[Token Tree $\mathcal{T}$]\label{def: token_tree} 
For a dataset $\mathcal{D} = \{ \vz_{i} \in \{0|1\}^{\infty} \mid i = 1,2, \cdots \}$, the token tree of $\mathcal{D}$, denoted as $\mathcal{T}_{\mathcal{D}}$, is defined as follows: each node has child nodes labeled \texttt{0} or \texttt{1}, along with an end-of-sequence (EOS) leaf node. The path from the root to a leaf node defines each response $\vz_{i}$, with the corresponding EOS node weight representing the response's probability while the weight of non-leaf nodes is the sum of the weights of their child nodes.
\end{definition}
\end{defbox}

In token tree modeling, the model's training process is conceptualized as learning the node weights of the token tree. However, due to the finite size of the model's parameters in practical scenarios, the model cannot accurately capture node weights at arbitrary depths of the token tree. Therefore, we propose the following assumption.

\begin{assumption}[Scale of $\mathcal{T}$ is Monotone with Model Size]\label{ass: size_to_node}
Consider a parameterized model $p_{\vtheta}\left(\cdot\right)$ and a dataset $\mathcal{D}$, We assume that the depth of the portion of $\mathcal{T}_{\mathcal{D}}$ that can be perfectly modelled by $p_{\vtheta}$ is monotonically increasing with the size of $\vtheta$.
\end{assumption}

Based on the above assumption, we can define the compression protocol during the model training and alignment process.

\begin{defbox}
\begin{definition}[The Compression Protocal]\label{def: compress_of_model}
Consider using the model $p_{\vtheta}\left(\cdot\right)$ to compress the dataset $\mathcal{D}$. The compression protocol is defined in two steps: a) Prune the token tree of $\mathcal{D}$, retaining only the top $d$ layers. b) Apply Huffman coding \cite{huffman1952method} to compress the pruned token tree. Specifically, each response from the root node to a leaf node is treated as a symbol in the Huffman coding alphabet, and the weight of the leaf node is the probability of the symbol.
\end{definition}
\end{defbox}

Due to the optimality and losslessness of Huffman coding \cite{huffman1952method}, the compression protocol ensures optimal compression while preserving losslessness. We can therefore calculate the model's ideal code length and other information-theoretic metrics based on the protocol.
\begin{thmbox}
\begin{theorem}[Ideal Code Length]\label{the: code_length}
Consider a finite parameter model $p_{\vtheta}\left(\cdot\right)$ training on dataset $\mathcal{D}$, the ideal code length $\mathcal{L}_{p_{\vtheta}}\left(\vx\right)$ of a random response $\vx$ compressed by $p_{\vtheta}$ can be expressed as:
\begin{align*}\mathbb{E}\left[\mathcal{L}_{p_{\vtheta}}\left(\vx\right)\right] = \left\lceil\frac{\big|\vx\big|}{d}\right\rceil\left\lceil-\sum_{l=1}^{d}\sum_{j=1}^{2^{l-1}}p_{lj}\log{p_{lj}}\right\rceil
\end{align*}
where $d$ represents the depth of the $\mathcal{T}_{\mathcal{D}}$ after pruning under Definition \ref{def: compress_of_model} protocol, and $p_{lj}$ represents the probability values of the leaf nodes for the $j$-th node at the $l$-th layer.
\end{theorem}
\end{thmbox}

Since training and alignment involve multiple datasets with different and independent distributions, we consider the joint compression scenario for multiple datasets. For $N$ pairwise disjoint datasets $\mathcal{D}_{1}, \cdots, \mathcal{D}_{N}$, the node weights $p_{lj}^{\mathcal{D}}$ of the token tree for the dataset $\mathcal{D} = \bigcup_{i=1}^{N} \mathcal{D}_{i}$ satisfy:
\begin{align*}
p_{lj}^{\mathcal{D}}=\frac{\sum_{i=1}^{N}p_{l}^{\mathcal{D}_{i}}\big|\mathcal{D}_{i}\big|}{\sum_{i=1}^{N}\big|\mathcal{D}_{i}\big|},
\end{align*}
where $p_{lj}^ {\mathcal{D}_{i}}$ stands for the probability value for nodes in $\mathcal{T}_{\mathcal{D}_{i}}$ while $\big|\mathcal{D}_{i}\big|$ represents the size of $\mathcal{D}_{i}$. Thus, the compression rate $\gamma_{p_{\vtheta}}^{\mathcal{D}_{i}}$ on specific datasets can also be defined accordingly:
\begin{align*}
    \gamma_{p_{\vtheta}}^{\mathcal{D}_{i}} &= \mathbb{E}_{\vx \sim \mathcal{P}_{i}}\left[ \frac{\mathbb{E}_{\vx\sim\mathcal{P}_{i}}\left[\mathcal{L}_{p_{\vtheta}}^{\mathcal{D}_{i}}(\vx) \right]}{\big| \vx \big|} \right] \\
    &= \Theta\left(-\frac{\sum_{l=1}^{d}\sum_{j=1}^{2^{l-1}}p_{lj}^{\mathcal{D}_{i}}\log{p_{lj}^{\mathcal{D}}}}{d}\right).
\end{align*}

Here, the compression rate is defined as the compressed encoding length divided by the original length \cite{deletang2023language} and ensures consistency between the training and compression objective, which means that minimizing the training loss is equivalent to minimizing the compression rate.
Please see Appendix \ref{app:theory_assumptions_and_proofs} for more details.

\subsection{The Formal Definition of \textit{Elasticity}}

To formalize \textit{inverse alignment} and \textit{elasticity}, we provide precise definitions of these concepts.
\begin{defbox}
\begin{definition}[\textit{Inverse Alignment}]
Given a language model \(p_{\vtheta_0}\), aligned on dataset \(\mathcal{D}_{a}\) to produce the aligned model \(p_{\vtheta_1}\). For any \(\epsilon\) > 0, if applying a dataset \(\mathcal{D}_{b}\) (where \(\vert \mathcal{D}_{b}\vert \ll \vert \mathcal{D}_{a}\vert\)) to \(p_{\vtheta_1}\) yields \(p_{\vtheta_0^{\prime}}\) such that \(\rho(p_{\vtheta_0^{\prime}}, p_{\vtheta_0}) \leq \epsilon\) for a given eval metric \(\rho\), we define the transition from \(p_{\vtheta_1}\) back to \(p_{\vtheta_0^{\prime}}\) as \textit{inverse alignment}.
\end{definition}
\end{defbox}

\begin{defbox}
\begin{definition}[The \textit{Elasticity} of LLMs]
Consider a language model $p_{\vtheta_0}$ and transformation $p_{\vtheta_0} \xmapsto{f(\mathcal{D}_a)} p_{\vtheta_1}$, \textit{elasticity} is said to exist in $\left(p_{\vtheta_0},\mathcal{D}_a\right)$ if there is an algorithmically simple \textit{inverse operation} \( g \) and a dataset \( \mathcal{D}_b \) such that \( \vert\mathcal{D}_b\vert \ll \vert\mathcal{D}_a\vert \), with the property that:
\[
p_{\vtheta_1} \xmapsto{g(\mathcal{D}_b)} p_{\vtheta_0^{\prime}}
\text{~and~}
\rho(p_{\vtheta_0^{\prime}}, p_{\vtheta_0}) \leq \epsilon_0.
\]
where $\epsilon_0$ is a constant.
\end{definition}
\end{defbox}
The eval metrics $\rho$ can be viewed as a measure of behavioral and distributional proximity between models. In the context of the compression theorem, we use compression rate $\gamma_{p_\vtheta}^{\mathcal{D}_i}$ as the metric $\rho$ to evaluate \textit{elasticity} during the alignment process.

\section{Why \textit{Elasticity} Affects Alignment?} \label{sec: 3}
 
By now, we hope to have convinced the reader of the concept of compression modeling in language models’ training and alignment. In the following, we apply this to analyze language models' training and alignment process, focusing on the underlying reasons that lead the model to alignment resistance.

In Figure \ref{fig: main}, we have already presented the theorem of \textit{elasticity}: when subject to fine-tuning perturbations, language models tend to retain the distribution associated with larger datasets while rejecting that of smaller ones. But why \textit{elasticity} resists alignment? Is there a corresponding \textit{elastic} invariant? How does \textit{elasticity} make \textit{inverse alignment} possible? In this section, we will formalize \textit{elasticity} of language models and analyze the impact of \textit{elasticity} on the model's behavior.

\subsection{Formal Derivation of \textit{Elasticity}}

Our primary goal is to investigate the behavioral changes of a language model after pre-training and alignment, particularly in response to perturbations, typically caused by fine-tuning with a minimal dataset. We use compression rate as an evaluation metric to assess these changes. In analyzing the model's behavior, we focus on two datasets: $\mathcal{D}_p$, a larger dataset representing pre-training or the primary objective of training, and $\mathcal{D}_a$, a smaller dataset representing the alignment process or the secondary objectives of training. Since the pre-training stage encompasses a wide range of data, without loss of generality, we assume that the datasets in the alignment process follow the same distribution as some subset of the pre-training data.

\begin{definition}[Normalized Compression Rate]
    For $N$ distinct datasets $\mathcal{D}_{1}, \cdots, \mathcal{D}_{N}$ and a parameter model $p_{\vtheta}$ compressing $\mathcal{D}=\bigcup_{i=1}^{N}\mathcal{D}_{i}$, the normalized compression rate $\gamma_{p_{\vtheta}}^{\mathcal{D}_{i}/\mathcal{D}}$ for a particular dataset $\mathcal{D}_{i}$ is defined as:
    \begin{equation}
\gamma_{p_{\vtheta}}^{\mathcal{D}_{i}/\mathcal{D}} = \gamma_{p_{\vtheta}}^{\mathcal{D}_{i}} - \log{M},
    \end{equation}
    where $M$ is the number of leaf nodes of the pruned tree $\mathcal{T}_{i}^{\prime}$ of dataset $\mathcal{D}_{i}$. 
\end{definition}

The normalized compression rate allows for comparing the model's compression performance across different datasets. The smaller the normalized compression rate of a dataset, the better the model's compression performance on the dataset.

With this definition in hand, we proceed to our main result: language models exhibit \textit{elasticity}.

\begin{thmbox}
\begin{theorem}[\textit{Elasticity} of Language Models]\label{theorem: main}
Consider the pre-training dataset $\mathcal{D}_p = \bigcup_{i=1}^{3}\mathcal{D}_{i}$, the alignment dataset $\mathcal{D}_a$, and the perturbation dataset $\mathcal{D}_t$, with the model $p_{\vtheta}(\cdot)$ trained on $\mathcal{D}=\mathcal{D}_{p} \cup \mathcal{D}_{a} \cup \mathcal{D}_{t}$. Assume that $\mathcal{D}_a \overset{d}{\sim} \mathcal{D}_{2}$, $\mathcal{D}_t \overset{d}{\sim} \mathcal{D}_{3}$, and $\mathcal{D}_{1},\mathcal{D}_{2},\mathcal{D}_{3}$ are each distributed according to a Pareto mass distribution \cite{newman2005power}. As the perturbation data volume $|\mathcal{D}_t|$ varies, $\gamma_{p_{\vtheta}}^{\mathcal{D}_p/\mathcal{D}}$ and $\gamma_{p_{\vtheta}}^{\mathcal{D}_a/\mathcal{D}}$ satisfy:
\begin{equation}
    \begin{split}
        \frac{d\gamma_{p_{\vtheta}}^{\mathcal{D}_a/\mathcal{D}}}{d\,l} &= \Theta \left( -k \frac{ d\gamma_{p_{\vtheta}}^{\mathcal{D}_p/\mathcal{D}}}{d\,l} \right),\\
    \frac{d\gamma_{p_{\vtheta}}^{\mathcal{D}_p/\mathcal{D}}}{d\,l} &< 0,
    \frac{d\gamma_{p_{\vtheta}}^{\mathcal{D}_a/\mathcal{D}}}{d\,l} > 0,
    \end{split}
\end{equation} \label{eq: main_theorem}
where $l = \frac{|\mathcal{D}_t|}{|\mathcal{D}_a|} \ll 1$, $k = \frac{|\mathcal{D}_p|}{|\mathcal{D}_a|} \gg 1$, and $\{\mathcal{D}_{i}\}_{i=1}^3 $ are datasets of equal cardinality.
\end{theorem}
\end{thmbox}

Theorem \ref{theorem: main} shows that as the perturbation increases, the normalized compression rates of the model for $\mathcal{D}_p$ decrease and $\mathcal{D}_a$ increase and the rate of changes is strongly correlated with the size of the datasets. Unlike the proportional changes in compression rates across different datasets, the language model seems to \textit{prefer} the dataset with a larger volume, leading to biased model behavior after the perturbation.

\subsection{\textit{Elasticity} and \textit{Inverse Alignment}}

Theorem \ref{theorem: main} shows the normalized compression rate change of different datasets is inversely proportional to their sizes under perturbations. A significant size difference between the pre-training and alignment datasets causes rapid performance degradation on the alignment dataset due to the inverse relationship, often spanning several orders of magnitude.

Intuitively, the model's compression process across multiple datasets is akin to resource allocation between towns: in a region with both a large metropolis and rural villages, to maximize overall economic productivity, resources are typically allocated to the metropolis first, to leverage its scale and agglomeration effects. In other words, the metropolis/large dataset occupies a dominant position in the system.

Therefore, the \textit{elasticity} of language models makes \textit{inverse alignment} possible. Due to the substantial data volume disparity between pre-training and alignment datasets, the model tends to revert to the pre-training rather than alignment distribution when subsequent perturbations occur. This indicates an inherent tendency toward \textit{inverse alignment} during fine-tuning. Maximizing the impact of \textit{elasticity} through well-designed perturbations will be key to achieving true \textit{inverse alignment}.

\subsection{\textit{Elasticity} and the Hooke's Law.}\label{sec: elasticity_mean}

The inverse proportionality result in Theorem \ref{theorem: main} provides a potential invariant in the model training and alignment process: after perturbation, the rate of change in the compression rates of different datasets is inversely proportional to their sizes, with the absolute value of the product being a constant. This constant characterizes the impact of the perturbation on the model and indirectly describes the model's resistance to perturbations, or its \textit{elasticity}.

The \textit{elasticity} of the model can be intuitively analogized to a series system of springs, as shown in Figure \ref{fig: main}. Consider two massless springs in series, with spring constants $\textcolor{forward}{k_1}$ and $\textcolor{forward}{k_2}$, respectively. When the entire system undergoes deformation due to an external force $F$, the system reaches a stable state, and the elastic force exerted by each spring is equal to $F$. According to Hooke's Law \cite{hooke2016lectures}, the elongation $\textcolor{inverse}{\Delta l_1}$ and $\textcolor{inverse}{\Delta l_2}$ of each spring is inversely proportional to its spring constant. Thus, in this system, we have:
\begin{equation*}
    F \propto \textcolor{forward}{k_1} \cdot\textcolor{inverse}{\Delta l_1} = \textcolor{forward}{k_2} \cdot\textcolor{inverse}{\Delta l_2} \;.
\end{equation*}

In the language model setting, after integrating Theorem \ref{theorem: main} to $l$, we obtain $\textcolor{inverse}{\Delta \gamma_{p_{\vtheta}}^{\mathcal{D}_i/\mathcal{D}}}$ across different datasets, which is equivalent to the change in the KL divergence $\textcolor{inverse}{\Delta D_{\text{KL}}(\mathcal{P}_{p_{\vtheta}}||\mathcal{P}_{\mathcal{D}_{i}})}$ between the model's distribution and the distributions of the individual datasets, is inversely proportional to the size of the datasets $\textcolor{forward}{|\mathcal{D}_i|}$. Here, we only consider the absolute value of $\Delta D_{\text{KL}}$. Analogous to the series spring model, the \textit{elasticity} $F$ in LLMs satisfies:
\begin{equation}
    F \propto \textcolor{forward}{|\mathcal{D}_i|} \cdot \textcolor{inverse}{\Delta D_{\text{KL}}(\mathcal{P}_{p_{\vtheta}}||\mathcal{P}_{\mathcal{D}_{i}})}\;,
\end{equation}

where $\textcolor{inverse}{\Delta D_{\text{KL}}}$ corresponds to $\textcolor{inverse}{\Delta l}$ in the spring model, while $\textcolor{forward}{|\mathcal{D}|}$ corresponds to the spring constant $\textcolor{forward}{k}$, thus leading to the \textit{elasticity} of LLMs.

\begin{table*}[ht]
\renewcommand{\arraystretch}{1}
\centering
\resizebox{0.85\textwidth}{!}{
\begin{tabular}{clccc}
\toprule
Datasets & Base Models & \inverse{$\vtheta_2 \rightarrow \vtheta_1$} \textit{vs.} \forward{$\vtheta_1 \rightarrow \vtheta_2$} & \inverse{$\vtheta_3 \rightarrow \vtheta_2$} \textit{vs.} \forward{$\vtheta_2 \rightarrow \vtheta_3$} & \inverse{$\vtheta_3 \rightarrow \vtheta_1$} \textit{vs.} \forward{$\vtheta_1 \rightarrow \vtheta_3$} \\
\midrule
\multirow{3}{*}{\textbf{Alpaca}} & Llama2-7B & 0.1589 \inverse{$\downarrow$} ~~ 0.2018 \forward{$\uparrow$} & 0.1953 \inverse{$\downarrow$} ~~  0.2143 \forward{$\uparrow$} & 0.1666 \inverse{$\downarrow$} ~~  0.2346 \forward{$\uparrow$} \\
& Llama2-13B & 0.1772 \inverse{$\downarrow$} ~~ 0.1958 \forward{$\uparrow$} & 0.2149 \inverse{$\downarrow$} ~~  0.2408 \forward{$\uparrow$} & 0.1835 \inverse{$\downarrow$} ~~  0.2345 \forward{$\uparrow$} \\
& Llama3-8B & 0.2540 \inverse{$\downarrow$} ~~ 0.2573 \forward{$\uparrow$} & 0.2268 \inverse{$\downarrow$} ~~  0.3229 \forward{$\uparrow$} & 0.2341 \inverse{$\downarrow$} ~~  0.2589 \forward{$\uparrow$} \\
\multirow{3}{*}{\textbf{Truthful}} & Llama2-7B & 0.1909 \inverse{$\downarrow$} ~~ 0.2069 \forward{$\uparrow$} & 0.1719 \inverse{$\downarrow$} ~~  0.1721 \forward{$\uparrow$} & 0.2011 \inverse{$\downarrow$} ~~  0.2542 \forward{$\uparrow$} \\
& Llama2-13B & 0.1704 \inverse{$\downarrow$} ~~ 0.1830 \forward{$\uparrow$} & 0.1544 \inverse{$\downarrow$} ~~  0.1640 \forward{$\uparrow$} & 0.1825 \inverse{$\downarrow$} ~~  0.2429 \forward{$\uparrow$} \\
& Llama3-8B & 0.2118 \inverse{$\downarrow$} ~~ 0.2256 \forward{$\uparrow$} & 0.2100 \inverse{$\downarrow$} ~~  0.2173 \forward{$\uparrow$} & 0.2393 \inverse{$\downarrow$} ~~  0.2898 \forward{$\uparrow$} \\
\multirow{3}{*}{\textbf{Safe}} & Llama2-7B & 0.2730 \inverse{$\downarrow$} ~~ 0.2809 \forward{$\uparrow$} & 0.2654 \inverse{$\downarrow$} ~~  0.2691 \forward{$\uparrow$} & 0.2845 \inverse{$\downarrow$} ~~  0.2883 \forward{$\uparrow$}\\
& Llama2-13B & 0.2419 \inverse{$\downarrow$} ~~ 0.2439 \forward{$\uparrow$} & 0.2320 \inverse{$\downarrow$} ~~  0.2327 \forward{$\uparrow$} & 0.2464 \inverse{$\downarrow$} ~~  0.2606 \forward{$\uparrow$} \\
& Llama3-8B & 0.2097 \inverse{$\downarrow$} ~~ 0.2156 \forward{$\uparrow$} & 0.2008 \inverse{$\downarrow$} ~~  0.2427 \forward{$\uparrow$} & 0.2277 \inverse{$\downarrow$} ~~  0.2709 \forward{$\uparrow$} \\
\bottomrule
\end{tabular}
}
\caption{\textbf{Comparsion between \inverse{\textit{inverse alignment}} and \forward{\textit{forward alignment}} training loss.} Under different model, task, and stage slicing settings, \forward{\textit{forward alignment}} is more challenging than \inverse{\textit{inverse alignment}}, \textit{i.e.,} the loss is higher. This indicates that pre-trained models tend to maintain their original distribution.
}
\label{tab:parameter_results}
\end{table*}

\section{How \textit{Elasticity} Resists Alignment?} \label{sec:4}

In the previous sections, we proved that LLMs have \textit{elasticity}. This section will analyze two specific phenomenons of it: 

\begin{itemize}
    \item \textbf{\resist{Resistance} for Pre-Trained Models.} Models tend to maintain the original distribution and resist alignment;
    \item \textbf{\rebound{Rebound} for Post-Trained Models.} Fine-tuning in the opposite direction of post-training (\textit{e.g.}, safe \textit{vs}. unsafe) causes post-trained models to return quickly to the pre-training distribution.
\end{itemize}

\subsection{Existence of Language Models' \resist{Resistance}}
\label{exp:resistance}

\begin{figure}[ht]
    \centering
    \includegraphics[width=\columnwidth]{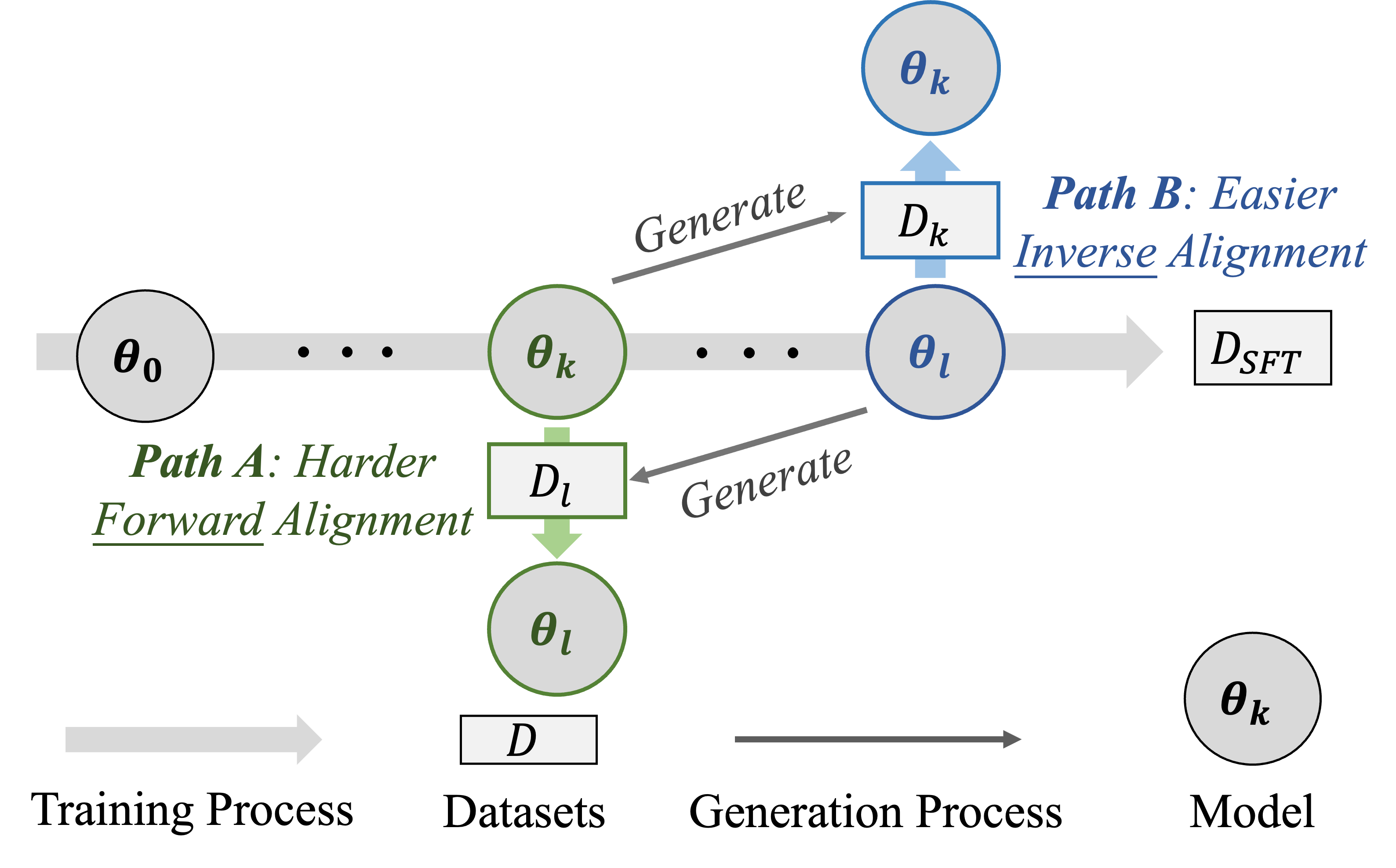}
    \caption{\textbf{Experiment pipeline for validating \resist{resistance}.} We conceptualize \resist{resistance} as: \inverse{\textit{inverse alignment}} is easier than \forward{\textit{forward alignment}}.}
    \label{fig:exp_1}
\end{figure}

\paragraph{Experiment Design.} We verify the existence of \resist{resistance} by arguing that \forward{\textit{forward alignment}} is harder than \inverse{\textit{inverse alignment}} for pre-trained models. Specifically, we first perform one epoch of SFT on a pre-trained LLM with parameters \(\vtheta_0\), saving the slices \(\{\vtheta_1, \vtheta_2, \ldots, \vtheta_n\}\). Subsequently, without loss of generality, we collect the responses of slices \(\vtheta_{k}\) and \(\vtheta_{l}\) (where $k < l$) on hold-out prompts, forming datasets \(D_{k}\) and \(D_{l}\). As shown in Figure \ref{fig:exp_1}, we define \forward{\textit{forward alignment} (\textbf{\textit{Path A}})} as the process of training \forward{\(\vtheta_{k}\)} on \(D_{l}\), and \inverse{\textit{inverse alignment} (\textbf{\textit{Path B}})} as the process of training \inverse{\(\vtheta_{l}\)} on \(D_{k}\).

\paragraph{Experiment Setup.} We select different SFT datasets including Alpaca \cite{taori2023stanford}, TruthfulQA \cite{lin2022truthfulqa}, and Beavertails \cite{ji2024beavertails, ji2024pku}, which correspond respectively to the model's 3H principle \cite{askell2021general}.  We divide them into three equal parts to obtain three corresponding slices \(\{\vtheta_1, \vtheta_2, \vtheta_3\}\). We consider Llama2-7B, Llama2-13B, and Llama3-8B \citep{touvron2023llama} as the base model $\vtheta_0$. 

\paragraph{Experiment Results.} As shown in Table \ref{tab:parameter_results}, the experimental results indicate that the training loss of \inverse{\textit{inverse alignment}} consistently remains lower than that of \forward{\textit{forward alignment}}, irrespective of the slice pair selection. This observation holds across all models and datasets in the experiments. All experimental results demonstrate that \inverse{\textit{inverse alignment}} is easier than \forward{\textit{forward alignment}} across diverse models and datasets and validate the existence of \resist{resistance}. To further verify the \textit{resistance} in language models, we provide additional experimental details and results in the Appendix \ref{app: exp_resistance}.

\subsection{Existence of \rebound{Rebound}}
\label{exp:rebound}

\begin{figure}[ht]
    \centering
    \includegraphics[width=\columnwidth]{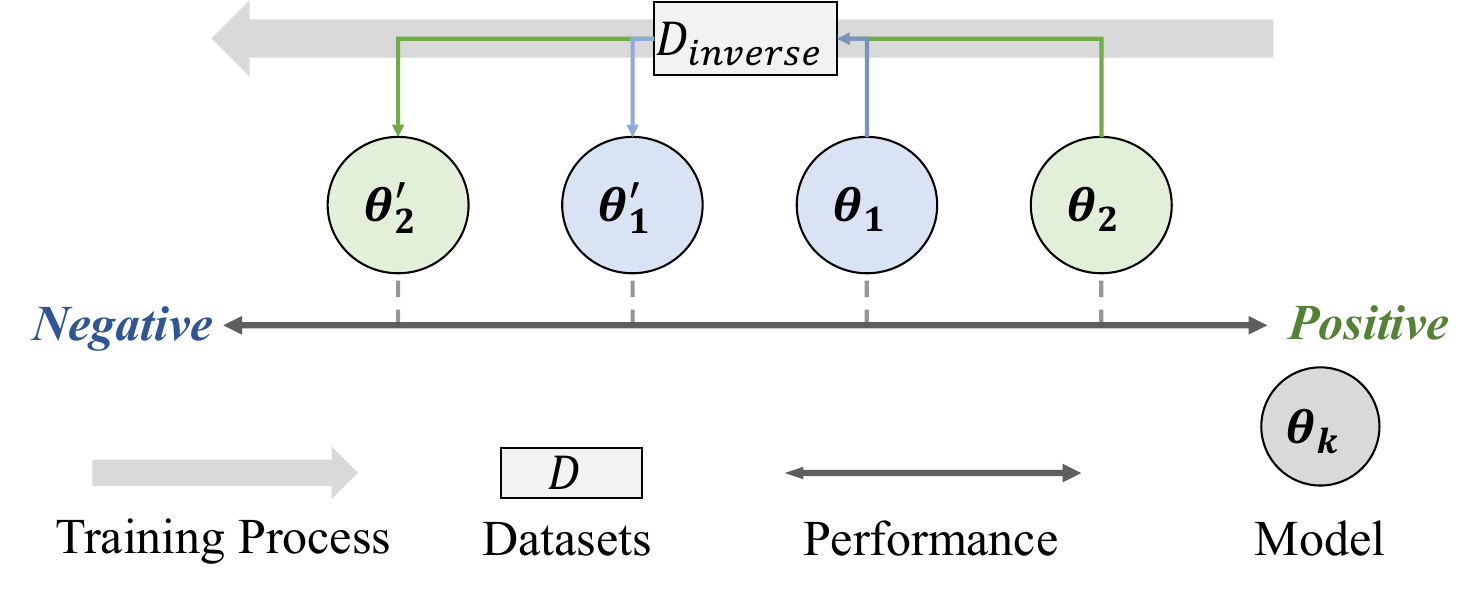}
    \caption{\textbf{Experiment pipeline for validating \rebound{rebound}.} We conceptualize \rebound{rebound} as: the more \forward{positive} the post-trained models' performance, the more \inverse{negative} it becomes after inverse finetuning.}
    \label{fig:exp_2}
\end{figure}

\begin{figure*}[ht]
    \centering\includegraphics[width=\textwidth]{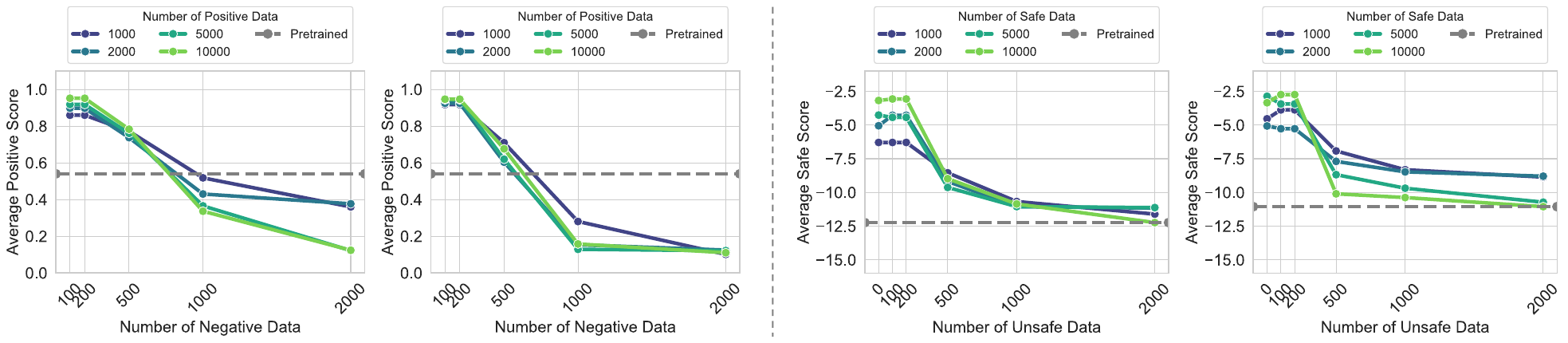}
    \caption{\textbf{Experimental results for validating the existence of \rebound{rebound} (left: IMDb, right: Beavertails).} The left part of each sub-figure is the performance of Gemma-2B while the right is Llama2-7B, respectively. Models trained with more \forward{positive} data initially perform better but perform worse after fine-tuning with \inverse{negative} data.}
    \label{exp2: existence}
\end{figure*}

\begin{figure*}[ht]
    \centering
    \begin{subfigure}{0.48\textwidth}
        \centering
            \includegraphics[width=\textwidth]{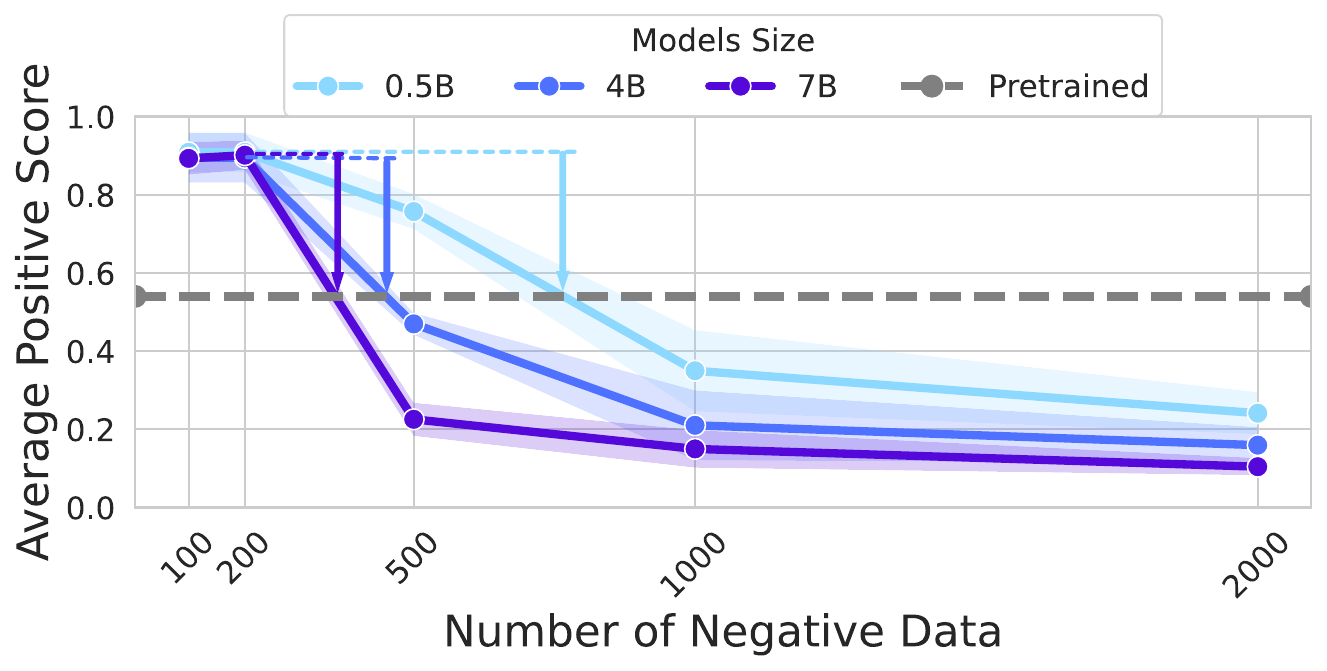}
    \end{subfigure}%
    \hfill
    \begin{subfigure}{0.48\textwidth}
        \centering
            \includegraphics[width=\textwidth]{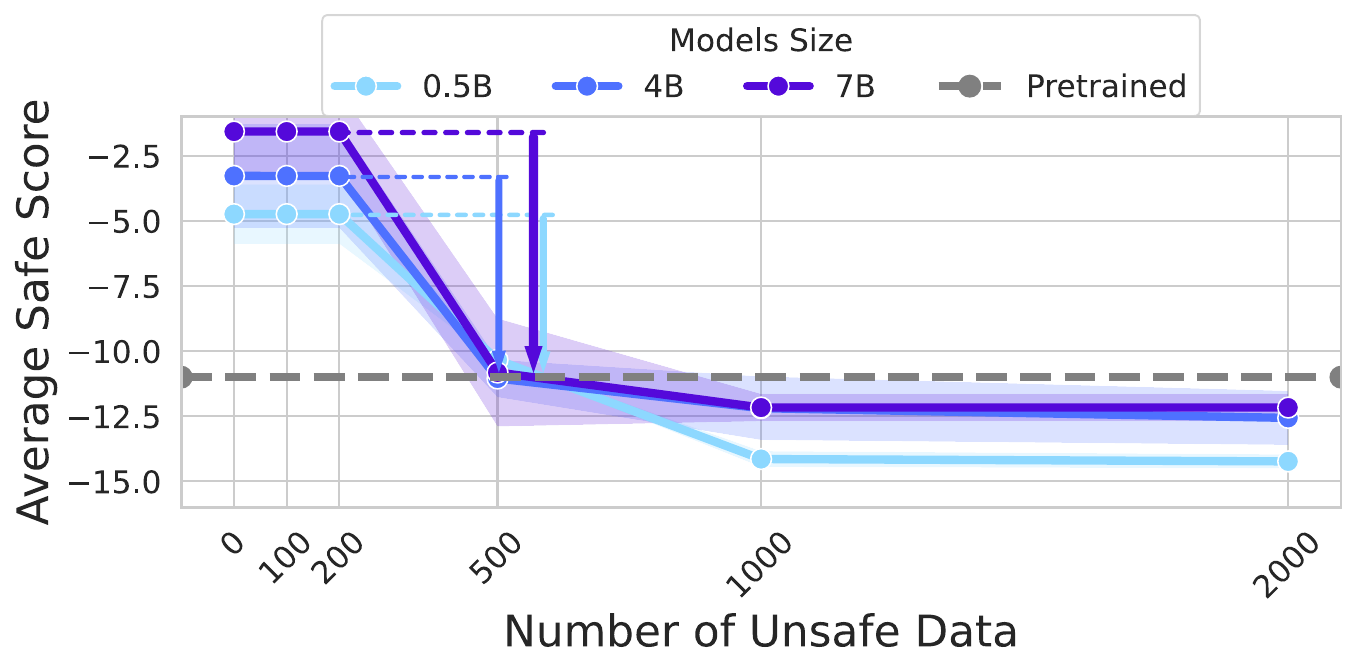}
    \end{subfigure}
    \caption{\textbf{Experimental results for validating \rebound{rebound} increases with model size (left: IMDb, right: Beavertails).} All single line covers positive data volume settings as Figure \ref{exp2: existence}, with shadow denoting std. As the model size increases, the performance of the aligned model deteriorates more rapidly after fine-tuning with \inverse{negative} data.}
    \label{exp2: model-size}
\end{figure*}

\begin{figure*}[ht]
    \centering
    \begin{subfigure}{0.48\textwidth}
        \centering
            \includegraphics[width=\textwidth]{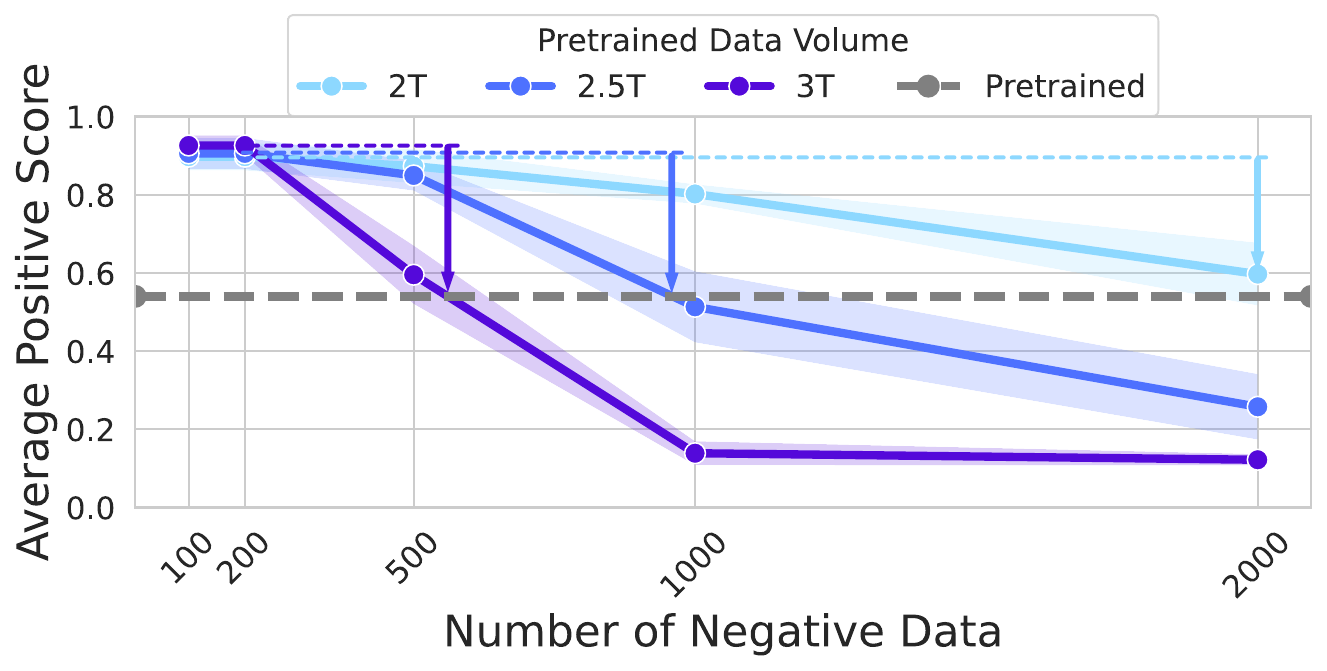}
    \end{subfigure}%
    \hfill
    \begin{subfigure}{0.48\textwidth}
        \centering
            \includegraphics[width=\textwidth]{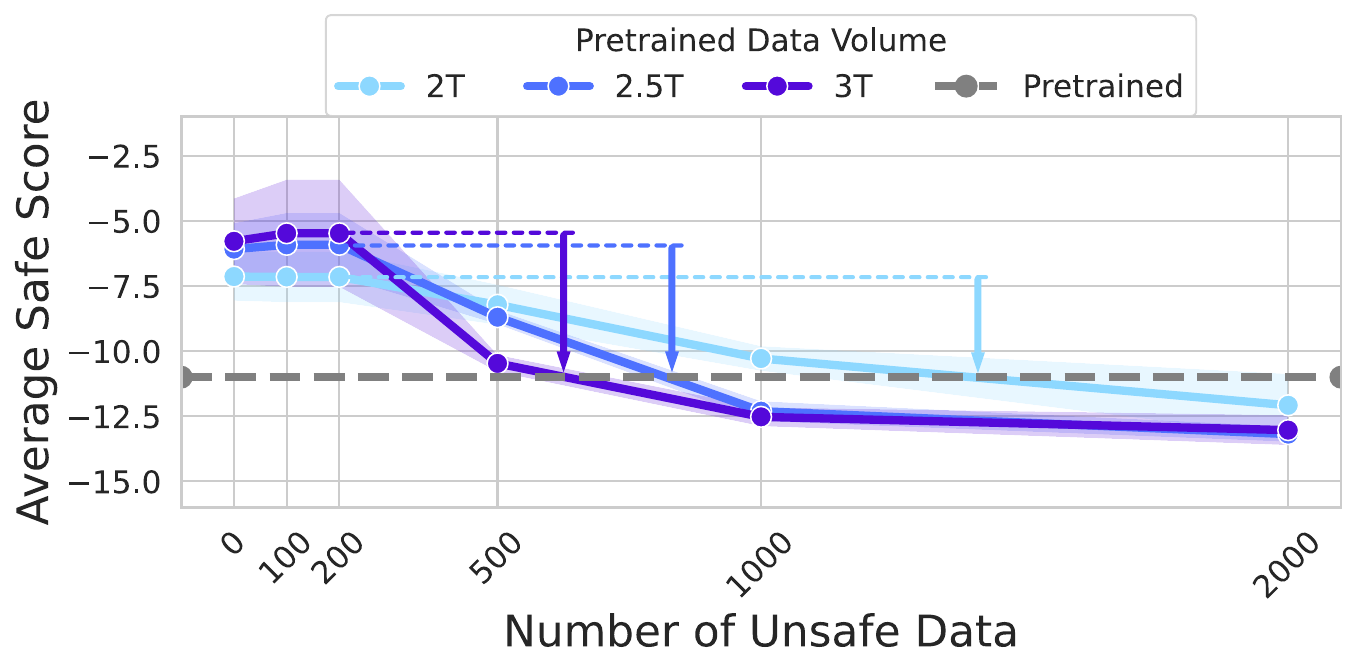}
    \end{subfigure}
    \caption{\textbf{Experimental results for validating \rebound{rebound} increases with pre-training data volume. (left: IMDb, right: Beavertails).} Each line covers positive data settings as Figure \ref{exp2: existence}, with shadow denoting std. As pre-training data volume increases, aligned model performance deteriorates more rapidly after fine-tuning with \inverse{negative} data.}
    \label{exp2: pre-train-data}
\end{figure*}

\paragraph{Experiment Design.} We verify the existence of \rebound{rebound} by arguing that for post-trained models, the more \forward{positive} the post-trained models' performance, the more \inverse{negative} it becomes after \textit{inverse alignment}. We validate tasks involving two opposing characteristics (\textit{e.g.} safe and unsafe). We first train slices \(\{\vtheta_1, \vtheta_2, ..., \vtheta_n\}\) based on a pre-trained model \(\vtheta_0\) using positive (\textit{e.g.,} safe) data of various volumes. Then we perform inverse finetuning on these slices using negative data (\textit{i.e.,} unsafe).

\paragraph{Tasks and Datasets.} We select two tasks: positive generation and single-turn safe conversation. For the former, we use the data classified as positive or negative in the IMDb dataset \citep{maas2011learning}. Referring to \cite{rafailov2024direct}, we use the first 2-8 tokens of each complete text as a prompt for LLMs to generate the subsequent content. For the latter, we use the data classified as safe and unsafe in Beavertails \citep{ji2024beavertails}. We organized the positive sample sizes into \{1000, 2000, 5000, 10000\}, while the negative sample sizes, were divided into \{100, 200, 500, 1000, 2000\}.

\paragraph{Evaluation and Metrics.} We collect the model's responses on the reserved test prompts. Then we use score models provided by existing research to complete the performance evaluation. For positive style generation, we refer to \cite{rafailov2024direct} and use the Sentiment Roberta model \citep{hartmann2023more} to classify the responses, taking the proportion of all responses classified as positive as the model score. For single-turn safe dialogue, we use the cost model \citep{dai2024safe} to score the safety of each response, using the average score of all responses as the model score.

\paragraph{Experiment Results.} We evaluate the \rebound{rebound} phenomenon on Llama2-7B \citep{touvron2023llama} and Gemma-2B \citep{team2024gemma}. The experimental results in Figure \ref{exp2: existence} show that, for models fine-tuned with a larger amount of positive sample data, their performance drops quicker under only a small amount of negative sample fine-tuning. Subsequently, the performance decline slows down and tends to stabilize. This result also confirms the previous conclusion: the initial rapid decline of model's performance is due to \rebound{rebound}, as the model is far from the pre-trained distribution; while the later stabilization of the countermeasure is due to \resist{resistance}, as the model is already close to the pre-trained distribution. 

To assess the generalizability of the \rebound{rebound} phenomenon, we perform additional ablation studies focusing on alignment algorithms, evaluation metrics, and fine-tuning directions. The results consistently validate the presence of the \rebound{rebound} phenomenon across language models. Further experimental details are available in Appendix \ref{app: add_exp_rebound}.

\subsection{Internal Factor of \rebound{Rebound}}
\label{exp:internal}

We conduct an in-depth analysis of the internal factors affecting the \rebound{rebound} phenomenon since it is crucial for the robustness of LLMs alignment. In particular, we investigate how model parameter scale and the amount of pre-training data influence \textit{rebound}. Owing to space constraints, further experimental analyses are included in Appendix \ref{app: add_exp_fac}.

\paragraph{\rebound{Rebound} Increases with Model Size.}  
To investigate how the \rebound{rebound} phenomenon varies with model size, we conducted experiments on Qwen models \citep{bai2023qwen} with parameter scales of 0.5B, 4B, and 7B. The experimental results in Figure \ref{exp2: model-size} show that as the model parameter size increases, the initial performance decline due to negative data fine-tuning is faster, while the subsequent decline is slower. This indicates that as the parameter size increases, there is an increase in \rebound{rebound} in response to both positive and negative data, further suggesting a positive correlation between model \textit{elasticity} and parameter scale. 

\paragraph{\rebound{Rebound} Increases with Pre-training Data Volume.} To verify that \rebound{rebound} increases with the growth of pre-training data, we vary pre-training slices (2.0T, 2.5T, and 3.0T) released by TinyLlama \citep{zhang2024tinyllama} and conduct the same experiment. As shown in Figure \ref{exp2: pre-train-data}, when the pre-training data volume increases, the initial performance decline due to negative data fine-tuning is faster, while the subsequent decline is slower. It demonstrates that larger pre-training data volumes reinforce the \rebound{rebound} of LLMs, which is consistent with the inference proposed in Theorem \ref{theorem: main}. 

\section{Conclusion and Outlook}
Our work uncovers a possible mechanism underlying the fragility of alignment: the \textit{elasticity} of language models. We demonstrate that upon perturbation, the normalized compression rate of language models changes inversely with dataset size, making them more inclined to retain pre-training distributions while forgetting alignment distributions, thus resisting alignment. Through extensive experiments, we validate the universality of this \textit{elasticity} effect, observing that the \textit{elasticity} strengthens as model size and pre-training data volume scale up. Fundamentally, our findings support a core assertion: language models exhibit \textit{elasticity}, and thereby inherently resist alignment.

\subsection{Limitations}
\label{sec:limitations_and_future_work}
Theory-wise, the primary limitation of our work is our specification of the \textit{mass distribution} (Assumption \ref{ass: pareto}), and empirical studies on the exact form of this distribution shall be valuable. Experiment-wise, we have not systematically validated \textit{elasticity} throughout the entire lifecycle of pre-training and alignment phases due to cost constraints. In future works, we plan to focus on whether this phenomenon is universally applicable, such as in multimodal models \citep{ji2024align, huh2024position}. Additionally, we aim to theoretically uncover the relationship between model \textit{elasticity} and \textit{scaling laws} \citep{kaplan2020scaling, xiao2024densing}, specifically determining the amount of training data required for \textit{elasticity} to manifest and quantitatively analyzing whether \textit{elasticity} intensifies as model parameters and pre-training data volume increase. 

\subsection{Broader Impacts}
\label{sec:broader_impact}

\paragraph{Rethinking Fine-tuning and Model Weight from Resist Alignment.}
Alignment fine-tuning methods aim to efficiently adjust the distribution of LLMs with minimal data to enhance their safety mechanisms. However, current alignment algorithms for language models may rely on optimization shortcuts, leading to local optima, while lacking involvement with the intrinsic mechanisms of the model \citep{qi2024safety}. From the perspective of language models' \textit{elasticity}, we require more robust alignment methods to ensure that modifications to model parameters extend beyond superficial changes \cite{qi2024safety, cohen2024rl}, thereby preventing the emergence of more effective \textit{inverse alignment} techniques that could give rise to extreme misalignment risks such as deceptive adversarial alignment \citep{IDAIS2024red, bengio2025international,marks2025auditing, baker2025monitoring} or alignment faking \citep{greenblatt2024alignment}. Although methods such as data cleansing during the training phase have proven to be an effective strategy for improving the malleability of a language model’s final distribution \cite{he2024s, qi2024finetuning}, such approaches are not particularly cost-effective and feasible. In Appendix \ref{app: algorithm_resist}, we provide a preliminary discussion of practical steps based on \textit{elasticity} to mitigate \textit{inverse alignment} risks, and we look forward to future research and practical developments grounded in \textit{elasticity}, which will enable the design of more robust alignment algorithms \cite{revel2024seal, zhang2024safe, chen2024robust, sheshadri2024targeted, liu2024rrm, li2025mixpro}, ultimately achieving true and reliable alignment for LLMs.

\paragraph{Rethinking Open-sourcing from Elasticity of LLMs.}
Open-sourcing is a double-edged sword \cite{seger2023open, anwar2024foundational, kukreja2024literature}. The public release of model weights enables the technical community to quickly identify and address potential vulnerabilities and conduct large-scale research on broad safety concerns \citep{touvron2023llama}. These efforts collectively help sustain a balanced offence-defense dynamic \citep{jervis1978cooperation, garfinkel2021does} in the open-source environment and improve model security \citep{eiras2024near}. However, the risk of misuse associated with open-sourcing models is inevitable \citep{seger2023open}. Malicious fine-tuning of open-source models \citep{urbina2022dual, sandbrink2023artificial} and the use of the open-source models to facilitate system jailbreaks \citep{zou2023universal} can pose serious threats to public safety \citep{goldstein2023generative, reuel2024open}. At present, the open-sourcing of powerful language models relies on safety alignment and rigorous security audits \citep{mokander2024auditing} to ensure responsible use. However, if advanced \textit{inverse alignment} techniques become feasible, even carefully aligned and audited open-source models could be reverted to their pre-alignment states at minimal cost, drastically lowering the barrier for jailbreaking aligned models and disrupting the offense-defense balance in the open-source ecosystem \cite{shapiro2010paper}. Therefore, it is crucial to develop alignment algorithms with fine-tuning robustness. We hope that deeper insights into the \textit{elasticity} mechanisms of language models will drive the advancement of untunable alignment methods, enabling models to maintain reliable safety throughout their entire lifecycle \cite{madiega2021artificial}.


\bibliography{custom}
\appendix
\onecolumn

\section{Assumptions and Proofs}
\label{app:theory_assumptions_and_proofs}

\begin{assumption}[Binary Tokens]
For simplicity and without loss of generality, we assume that all datasets share a uniform token table containing only binary tokens (specifically \texttt{0/1}).
\end{assumption}

\begin{thmbox}
\begin{theorem}[Ideal Code Length]
Consider a finite parameter model $p_{\vtheta}\left(\cdot\right)$ training on dataset $\mathcal{D}$, the ideal code length $\mathcal{L}_{p_{\vtheta}}\left(\vx\right)$ of a random response $\vx$ compressed by $p_{\vtheta}$ can be expressed as follows:
\begin{align*}
     \mathbb{E}\left[\mathcal{L}_{p_{\vtheta}}\left(\vx\right)\right] = \left\lceil\frac{\big|\vx\big|}{d}\right\rceil\left\lceil-\sum_{l=1}^{d}\sum_{j=1}^{2^{l-1}}p_{lj}\log{p_{lj}}\right\rceil,
\end{align*}
where $d$ represents the depth of the $\mathcal{T}_{\mathcal{D}}$ after pruning under Definition \ref{def: compress_of_model} protocol, $p_{lj}$ represents the probability values of the \texttt{EOS} nodes for the $j$-th node at the $l$-th layer.
\end{theorem}
\end{thmbox}

\begin{proof}
    When $\big|\vx\big|\leq d$, the compression protocol defined in Definition \ref{def: compress_of_model} can perfectly compress $\vx$. Hence, the expectation of the ideal code length $\mathcal{L}_{p_{\vtheta}}\left(\vx\right)$ satisfies:
    \begin{equation*}
        \mathbb{E}\left[\mathcal{L}_{p_{\vtheta}}\left(\vx\right)\right] = \left\lceil-\sum_{l=1}^{d}\sum_{j=1}^{2^{l-1}}p_{lj}\log{p_{lj}}\right\rceil,
    \end{equation*}
    where $d$ represents the depth of the pruned tree $\mathcal{T}_{\mathcal{D}}^{\prime}$ and $p_{lj}$ represents the probability values of the \texttt{EOS} nodes for the $j$-th node at the $l$-th layer.
    
    Now consider $sd\leq\big|\vx\big|\leq (s+1)d$. Let us suppose that $\vx=(\vx_1\cdots\vx_s\vx_{s+1})$, where $|\vx_k|=d$, for $k\in \{1,\ldots,s\}$ and $|\vx_{s+1}|\leq d$. In this case, $\vx$ cannot be perfectly compressed by the model. Hence, the compression of $x$ needs to be performed in segments, and the length of each segment is not greater than $d$. 
    \begin{align}
\mathbb{E}_{\vx}\left[\mathcal{L}_{p_{\vtheta}}\left(\vx\right)\right] &= \mathbb{E}_{\vx_1}\left[\mathcal{L}_{p_{\vtheta}}\left(\vx_1\right)\right] + \mathbb{E}_{\vx_1}\mathbb{E}_{(\vx_2\ldots\vx_{s+1})}\left[\mathcal{L}_{p_{\vtheta}}\left(\left(\vx_2\ldots\vx_{s+1}\right)\right)\big|\vx_1\right]\nonumber\\
&=\mathbb{E}_{\vx_1}\left[\mathcal{L}_{p_{\vtheta}}\left(\vx_1\right)\right] + \sum_{\vx_1} p(\vx_1)\sum_{(\vx_2\ldots\vx_{s+1})}p(\vx_2\ldots\vx_{s+1}\big|\vx_1)\cdot\mathcal{L}_{p_{\vtheta}}\left((\vx_2\ldots\vx_{s+1})\right)\nonumber\\
&=\mathbb{E}_{\vx_1}\left[\mathcal{L}_{p_{\vtheta}}\left(\vx_1\right)\right] + \mathbb{E}_{(\vx_2\ldots\vx_{s+1})}\left[\mathcal{L}_{p_{\vtheta}}\left(\left(\vx_2\ldots\vx_{s+1}\right)\right)\right]\nonumber\\
&=\sum_{k=1}^{s+1}\mathbb{E}_{\vx_k}\left[\mathcal{L}_{p_{\vtheta}}\left(\vx_k\right)\right]\nonumber\\
&=\left\lceil\frac{\big|\vx\big|}{d}\right\rceil\left\lceil-\sum_{l=1}^{d}\sum_{j=1}^{2^{l-1}}p_{lj}\log{p_{lj}}\right\rceil,  \nonumber
\end{align}
thus the proof is completed. 
\end{proof}

\begin{definition}[Mass Distribution in Token Tree]\label{def: mass}
Consider the sample space $\Omega$ consisting of all responses in dataset $\mathcal{D}$. The probability distribution $\mathcal{P}_{\mathcal{D}}$ of all subtrees at the $d$-th level nodes of $\mathcal{T}_{\mathcal{D}}$ is a mapping from $\Omega$ to $[0,1]$. Let $X_{\mathcal{D}}$ be the random variable representing the probability value taken at each leaf multiplied by the number of leaves. The mass distribution $P_{mass}$ represents the probability that $X_{\mathcal{D}}$ takes the corresponding probability value. According to the definition of $\mathcal{P}_{mass}$, $\mathbb{E}[X_\mathcal{D}]=1$ .
\end{definition}

\begin{remark}[Mixture of Mass Distribution]
 For independently and differently distributed datasets $\mathcal{D}_1, \ldots, \mathcal{D}_{N}$, $\mathcal{D}=\bigcup_{i=1}^{N}\mathcal{D}_{i}$ is a mixture of these datasets. For the pruned trees $\mathcal{T}_{1},\ldots,\mathcal{T}_{N}$ of these datasets with depth $d$, the random variables of their leaf nodes satisfy the following relationship:
\begin{equation*}
    X_{\mathcal{D}} = \frac{\sum_{k=1}^{N}\big|\mathcal{D}_{k}\big|X_{\mathcal{D}_{k}}}{\sum_{k=1}^{N}\big|\mathcal{D}_{k}\big|},
\end{equation*}
where $X_{\mathcal{D}_{k}}$ follows the mass distribution $\mathcal{P}_{mass}^{k}$. For $X_{\mathcal{D}_{k_i}}$ and $X_{\mathcal{D}_{k_j}}$ from different datasets, $X_{\mathcal{D}_{k_i}}$ and $X_{\mathcal{D}_{k_j}}$ are independent of each other.

\end{remark}

\begin{lemma}[Entropy of Mass Distribution]\label{lemma: entropy_mass}
Consider the pruned trees $\mathcal{T}_{k}^{\prime}$ and $\mathcal{T}^{\prime}$ of dataset $\mathcal{D}_k$ and $\mathcal{D}=\bigcup_{i=1}^{N}\mathcal{D}_{i}$. Denote that the response distribution and the mass distribution of $\mathcal{T}^{\prime}$'s tree nodes are $\mathcal{P}^\mathcal{D}$ and $\mathcal{P}_{mass}$. Similarly, we define the response distribution $\mathcal{P}^\mathcal{D}_{k}$ and the mass distribution $\mathcal{P}_{mass}^{k}$. When the number of the leaf nodes $M$ is sufficiently large, the cross-entropy of the response distribution can be rewritten as follows.
\begin{equation*}
    \mathbb{E}_{\vx\sim\mathcal{P}_{k}}\left[-p^{\mathcal{D}_{k}}\log{p^{\mathcal{D}}}\right] = \mathbb{E}_{X_{\mathcal{D}_{k}}\sim\mathcal{P}_{mass}^{k},X_{\mathcal{D}}\sim\mathcal{P}_{mass}}\left[-X_{\mathcal{D}_{k}}\log{X_{\mathcal{D}}}\right] + \log{M},
\end{equation*}
where $p^{\mathcal{D}}$, $p^{\mathcal{D}_k}$ stand for the probability of the leaf nodes of $\mathcal{T}^{\prime}$,$\mathcal{T}_{k}^{\prime}$ while $X_{\mathcal{D}_k}$, $X_{\mathcal{D}}$ stand for the random variables of the probability of the leaf nodes in $\mathcal{T}^{\prime}$,$\mathcal{T}_{k}^{\prime}$, respectively.
\end{lemma}
\begin{proof}
Let $M$ be the number of leaf nodes of $\mathcal{T}^{\prime}$. According to the definitions of the response distribution $\mathcal{P}$ and mass distribution $\mathcal{P}_{mass}$, we have $Mp^{\mathcal{D}_j}=X_{\mathcal{D}_j}, \forall j \in\{1,\dots,N\}$. 
Therefore,
\begin{align}
    \mathbb{E}_{\vx\sim\mathcal{P}_{k}}\left[-p^{\mathcal{D}_{k}}\log{p^{\mathcal{D}}}\right] &= \sum_{i=1}^{M}-p_i^{\mathcal{D}_k}\log{p_i^{\mathcal{D}}}\nonumber \\
    &=\sum_{i=1}^{M}-\frac{1}{M}X_{i,\mathcal{D}_{k}}\log{X_{i,\mathcal{D}}} + \log{M}\nonumber\\
    &=\mathbb{E}_{X_{\mathcal{D}_{k}}\sim\mathcal{P}_{mass}^{k},X_{\mathcal{D}}\sim\mathcal{P}_{mass}}\left[-X_{\mathcal{D}_{k}}\log{X_{\mathcal{D}}}\right] + \log{M}. \nonumber
\end{align}
\end{proof}

\begin{remark}
    In Lemma \ref{lemma: entropy_mass}, $X_{\mathcal{D}_k}$ are assumed to be independent. However due to $\sum_{i=1}^{M}p_i^{\mathcal{D}_k}=1$, the $X_{\mathcal{D}_k}$ are not actually independent. Considering that $M$ is sufficiently large in our subsequent analysis, we can regard the independence of $X_{\mathcal{D}_k}$ as a good approximation.

\end{remark}

We assume that the mass distribution of the segment follows a heavy-tailed Pareto distribution, with supporting evidence from \citet{mingard2021sgd, zipf1946psychology}.
\begin{assumption}[The Pareto Distribution]\label{ass: pareto}
We assume that the mass distribution of pruned token trees $\mathcal{T}$ of depth $d$, across different datasets, follows a Pareto distribution with identical parameters:
\begin{equation*}
    p_{X}(x)=\begin{cases}
         \frac{\alpha c^{\alpha}}{x^{\alpha+1}} & x \geq c, \\
0 & x < c,
    \end{cases}
\end{equation*}
where $\alpha, c$ are parameters of the Pareto distribution. 
\end{assumption}

\begin{thmbox}
\begin{theorem}[\textit{Elasticity} of Language Models] Consider the pre-training dataset $\mathcal{D}_p = \bigcup_{i=1}^{3}\mathcal{D}_{i}$, the alignment dataset $\mathcal{D}_a$, and the perturbation dataset $\mathcal{D}_t$, with the model $p_{\vtheta}(\cdot)$ trained on $\mathcal{D}=\mathcal{D}_{p} \cup \mathcal{D}_{a} \cup \mathcal{D}_{t}$. Assume that $\mathcal{D}_a \overset{d}{\sim} \mathcal{D}_{2}$, $\mathcal{D}_t \overset{d}{\sim} \mathcal{D}_{3}$, and $\mathcal{D}_{1},\mathcal{D}_{2},\mathcal{D}_{3}$ are each distributed according to a Pareto mass distribution \cite{newman2005power}. As the perturbation data volume $|\mathcal{D}_t|$ varies, $\gamma_{p_{\vtheta}}^{\mathcal{D}_p/\mathcal{D}}$ and $\gamma_{p_{\vtheta}}^{\mathcal{D}_a/\mathcal{D}}$ satisfy:
\begin{equation}
    \begin{split}
        \frac{d\gamma_{p_{\vtheta}}^{\mathcal{D}_a/\mathcal{D}}}{d\,l} &= \Theta \left( -k \frac{ d\gamma_{p_{\vtheta}}^{\mathcal{D}_p/\mathcal{D}}}{d\,l} \right),\\
    \frac{d\gamma_{p_{\vtheta}}^{\mathcal{D}_p/\mathcal{D}}}{d\,l} &< 0,
    \frac{d\gamma_{p_{\vtheta}}^{\mathcal{D}_a/\mathcal{D}}}{d\,l} > 0,
    \end{split}
\end{equation} 
where $l = \frac{|\mathcal{D}_t|}{|\mathcal{D}_a|} \ll 1$, $k = \frac{|\mathcal{D}_p|}{|\mathcal{D}_a|} \gg 1$, and $\{\mathcal{D}_{i}\}_{i=1}^3 $ are datasets of equal cardinality.
\end{theorem}
\end{thmbox}

\begin{proof}
For the sake of convenience in calculations, we first use Lemma \ref{lemma: entropy_mass} to replace the Shannon entropy of response distribution.
\begin{align}
    \frac{d\gamma_{p_{\vtheta}}^{\mathcal{D}_j/\mathcal{D}}}{d\,l} &= \frac{d\left(\mathbb{E}_{\vx\sim\mathcal{P}_{j}}\left[-p^{\mathcal{D}_{j}}\log{p^{\mathcal{D}}}\right]-\log{M}\right)}{d\,l}\nonumber\\
    &= \frac{d\left(\mathbb{E}_{X_{\mathcal{D}_{j}}\sim\mathcal{P}_{mass}^{j},X_{\mathcal{D}}\sim\mathcal{P}_{mass}}\left[-X_{\mathcal{D}_{j}}\log{X_{\mathcal{D}}}\right]\right)}{d\,l}.\nonumber
\end{align}
According to Assumption \ref{ass: pareto}, $X_{D_{j}}$ follows a Pareto distribution with the parameters $\alpha$ and $c$. Hence,
\begin{align}  &\mathbb{E}_{X_{\mathcal{D}_{p}}\sim\mathcal{P}_{mass}^{p},X_{\mathcal{D}}\sim\mathcal{P}_{mass}}\left[-X_{\mathcal{D}_{p}}\log{X_{\mathcal{D}}}\right] \nonumber\\
    &= -\int_{c}^{+\infty} \int_{c}^{+\infty} \int_{c}^{+\infty} \frac{\alpha^3 c^{3\alpha}}{\prod_{i=1}^3x_{i}^{\alpha+1}}\cdot\frac{x_1+x_2+x_3}{3}\cdot\log{\frac{\sum_{i \in \{p,a,t\}}\big|\mathcal{D}_{i}\big|x_{i}}{\sum_{i \in \{p,a,t\}}\big|\mathcal{D}_{i}\big|}}dx_1dx_2dx_3\nonumber\\
    &= -\int_{c}^{+\infty} \int_{c}^{+\infty} \int_{c}^{+\infty} \frac{\alpha^3 c^{3\alpha}(x_1+x_2+x_3)}{3\prod_{i=1}^3x_{i}^{\alpha+1}}\log{\frac{\frac{k}{3}(x_{1}+x_2+x_3)+x_{2}+lx_3}{k+l+1}}dx_1dx_2dx_3\nonumber \\
&\mathbb{E}_{X_{\mathcal{D}_{a}}\sim\mathcal{P}_{mass}^{a},X_{\mathcal{D}}\sim\mathcal{P}_{mass}}\left[-X_{\mathcal{D}_{a}}\log{X_{\mathcal{D}}}\right] \nonumber\\
    &= -\int_{c}^{+\infty} \int_{c}^{+\infty} \int_{c}^{+\infty} \frac{\alpha^3 c^{3\alpha}x_2}{\prod_{i=1}^3x_{i}^{\alpha+1}}\log{\frac{\frac{k}{3}(x_{1}+x_2+x_3)+x_{2}+lx_3}{k+l+1}}dx_1dx_2dx_3\nonumber
\end{align}

where $j=1,2,3$. Therefore, $\frac{d\gamma_{p_{\vtheta}}^{\mathcal{D}_p/\mathcal{D}}}{d\,l}$ and $\frac{d\gamma_{p_{\vtheta}}^{\mathcal{D}_a/\mathcal{D}}}{d\,l}$ can be written as:
\begin{align}\label{eq: gamma}
\frac{d\gamma_{p_{\vtheta}}^{\mathcal{D}_p/\mathcal{D}}}{d\,l} = -\frac{d S_1}{d\;l},\;\;\;\;\frac{d\gamma_{p_{\vtheta}}^{\mathcal{D}_a/\mathcal{D}}}{d\,l} = -\frac{d S_2}{d\;l}\,
\end{align}
where
\begin{align}
S_1=&\int_{c}^{+\infty} \int_{c}^{+\infty} \int_{c}^{+\infty} \frac{\alpha^3 c^{3\alpha}(x_1+x_2+x_3)}{3x_{1}^{\alpha+1} x_{2}^{\alpha+1} x_{3}^{\alpha+1}}\log{\frac{\frac{k}{3}(x_{1}+x_2+x_3)+x_{2}+lx_3}{k+l+1}}dx_1dx_2dx_3\nonumber\\
S_2=&\int_{c}^{+\infty} \int_{c}^{+\infty} \int_{c}^{+\infty} \frac{\alpha^3 c^{3\alpha}}{x_{1}^{\alpha+1} x_{2}^{\alpha} x_{3}^{\alpha+1}}\log{\frac{\frac{k}{3}(x_{1}+x_2+x_3)+x_{2}+lx_3}{k+l+1}}dx_1dx_2dx_3.\nonumber
\end{align}
Proving that $\frac{d\gamma_{p_{\vtheta}}^{\mathcal{D}_a/\mathcal{D}}}{d\,l} = \Theta \left( -k \frac{d\gamma_{p_{\vtheta}}^{\mathcal{D}_p/\mathcal{D}}}{d\,l} \right)$ is equivalent to proving:
\begin{equation}\label{eq: lim}
    \lim_{k \to +\infty,\,l \to 0}\frac{k\cdot \frac{ d\gamma_{p_{\vtheta}}^{\mathcal{D}_p/\mathcal{D}}}{d\,l}+\frac{ d\gamma_{p_{\vtheta}}^{\mathcal{D}_a/\mathcal{D}}}{d\,l}}{k\cdot\frac{d\gamma_{p_{\vtheta}}^{\mathcal{D}_p/\mathcal{D}}}{d\, l}} = 0,
\end{equation}
By substituting (\ref{eq: gamma}) into (\ref{eq: lim}), we have
\begin{align}\label{eq: main}
    \lim_{k \to +\infty,\,l \to 0}\frac{k\cdot \frac{ d\gamma_{p_{\vtheta}}^{\mathcal{D}_p/\mathcal{D}}}{d\,l} + \frac{ d\gamma_{p_{\vtheta}}^{\mathcal{D}_a/\mathcal{D}}}{d\,l} }{k\cdot\frac{d\gamma_{p_{\vtheta}}^{\mathcal{D}_p/\mathcal{D}}}{d\, l}}
    =\frac{\lim_{k \to +\infty,\,l \to 0}\left( \frac{k \cdot d S_1 + d S_2}{d\;l} \right)}{\lim_{k \to +\infty,\,l \to 0}\frac{k\cdot d S_1}{d\;l}}.
\end{align}

Now calculate the values of $\frac{dS_1}{d \;l}$ for the case when $l \to 0$ and $k$ is sufficiently large.
\begin{align}
    &\lim_{l \to 0}\frac{d S_1}{d\;l} \nonumber\\
    =& \alpha^3c^{3\alpha}\lim_{l \to 0} \int_{c}^{+\infty} \int_{c}^{+\infty} \int_{c}^{+\infty} \frac{(x_1 + x_2 + x_3)}{3x_{1}^{\alpha+1} x_{2}^{\alpha+1} x_{3}^{\alpha+1}}\nonumber\\
    &\quad\cdot\frac{(k+1+l)x_3 - (\frac{k}{3}(x_{1}+x_2+x_3)+x_{2}+lx_3)}{(k+1+l)(\frac{k}{3}(x_{1}+x_2+x_3)+x_{2}+lx_3)}dx_1dx_2dx_3 \nonumber\\
    =& \alpha^3c^{3\alpha}\lim_{ l \to 0} \int_{c}^{+\infty} \int_{c}^{+\infty} \int_{c}^{+\infty} \frac{(x_1 + x_2 + x_3)x_3}{3x_{1}^{\alpha+1} x_{2}^{\alpha+1} x_{3}^{\alpha+1}(\frac{k}{3}(x_{1}+x_2+x_3)+x_{2}+lx_3)}dx_1dx_2dx_3 \nonumber\\
    &\quad - \alpha^3c^{3\alpha}\lim_{l \to 0} \int_{c}^{+\infty} \int_{c}^{+\infty} \int_{c}^{+\infty} \frac{(x_1 + x_2 + x_3)}{3x_{1}^{\alpha+1} x_{2}^{\alpha+1} x_{3}^{\alpha+1}(k+1+l)}dx_1dx_2dx_3 \nonumber\\
    =& \alpha^3c^{3\alpha} \int_{c}^{+\infty} \int_{c}^{+\infty} \int_{c}^{+\infty} \lim_{l \to 0} \frac{(x_1 + x_2 + x_3)x_3}{3x_{1}^{\alpha+1} x_{2}^{\alpha+1} x_{3}^{\alpha+1}(\frac{k}{3}(x_{1}+x_2+x_3)+x_{2}+lx_3)}dx_1dx_2dx_3 \nonumber\\
    &\quad - \alpha^3c^{3\alpha}\lim_{l \to 0} \int_{c}^{+\infty} \int_{c}^{+\infty} \int_{c}^{+\infty} \frac{x_1}{x_{1}^{\alpha+1} x_{2}^{\alpha+1} x_{3}^{\alpha+1}(k+1+l)}dx_1dx_2dx_3 \nonumber\\
    =& \alpha^3c^{3\alpha} \int_{c}^{+\infty} \int_{c}^{+\infty} \int_{c}^{+\infty} \frac{x_1 + x_2 + x_3}{3x_{1}^{\alpha+1} x_{2}^{\alpha+1} x_{3}^{\alpha+1}} \nonumber\\
    &\quad\cdot \lim_{l \to 0}\left[\frac{x_3}{\frac{k}{3}(x_{1}+x_2+x_3)+x_{2}}- \frac{x_3^2}{[\frac{k}{3}(x_{1}+x_2+x_3)+x_{2}]^2} l\right]dx_1dx_2dx_3 \nonumber\\
    &\quad - \alpha^3c^{3\alpha}\lim_{l \to 0} \int_{c}^{+\infty} \int_{c}^{+\infty} \int_{c}^{+\infty} \frac{x_1}{x_{1}^{\alpha+1} x_{2}^{\alpha+1} x_{3}^{\alpha+1}(k+1+l)}dx_1dx_2dx_3\label{eq: taylor_expansion_l} \\
    =& \alpha^3c^{3\alpha} \int_{c}^{+\infty} \int_{c}^{+\infty} \int_{c}^{+\infty} \frac{x_1 + x_2 + x_3}{3x_{1}^{\alpha+1} x_{2}^{\alpha+1} x_{3}^{\alpha+1}} \frac{x_3}{\frac{k}{3}(x_{1}+x_2+x_3)+x_{2}}dx_1dx_2dx_3 \nonumber\\
    &\quad - \alpha^3c^{3\alpha}\lim_{ l \to 0} \int_{c}^{+\infty} \int_{c}^{+\infty} \int_{c}^{+\infty} \frac{x_1}{x_{1}^{\alpha+1} x_{2}^{\alpha+1} x_{3}^{\alpha+1}(k+1+l)}dx_1dx_2dx_3 \nonumber\\
    =& \alpha^3c^{3\alpha} \int_{c}^{+\infty} \int_{c}^{+\infty} \int_{c}^{+\infty} \frac{x_1 + x_2 + x_3}{3x_{1}^{\alpha+1} x_{2}^{\alpha+1} x_{3}^{\alpha+1}} \left[\frac{x_3}{(x_{1}+x_2+x_3)}\frac{3}{k} - \frac{x_2x_3}{(x_1+x_2+x_3)^2}\frac{9}{k^2}\right]dx_1dx_2dx_3 \nonumber\\
    &\quad - \alpha^3c^{3\alpha}\lim_{l \to 0} \int_{c}^{+\infty} \int_{c}^{+\infty} \int_{c}^{+\infty} \frac{x_1}{x_{1}^{\alpha+1} x_{2}^{\alpha+1} x_{3}^{\alpha+1}(k+1+l)}dx_1dx_2dx_3 \label{eq: taylor_expansion_k}\\
    =& \alpha^3c^{3\alpha} \int_{c}^{+\infty} \int_{c}^{+\infty} \int_{c}^{+\infty} \frac{x_1}{x_{1}^{\alpha+1} x_{2}^{\alpha+1} x_{3}^{\alpha+1}}\frac{1}{k(k+1)}dx_1dx_2dx_3 \nonumber\\
    &\quad - \alpha^3c^{3\alpha} \int_{c}^{+\infty} \int_{c}^{+\infty} \int_{c}^{+\infty} \frac{1}{x_{1}^{\alpha+1} x_{2}^{\alpha+1} x_{3}^{\alpha+1}}\frac{3x_2x_3}{k^2(x_1+x_2+x_3)}dx_1dx_2dx_3
\end{align}
Here, we apply a Taylor expansion to Equation (\ref{eq: taylor_expansion_l}) and Equation (\ref{eq: taylor_expansion_k}) with respect to $l$ and $\frac{1}{k}$, respectively. 

Similarly, we can calculate the value of $\frac{d S_2}{d\;l}$.
\begin{align}
    &\lim_{l \to 0}\frac{d S_2}{d\;l} \nonumber\\
    =& \alpha^3c^{3\alpha}\lim_{l \to 0} \int_{c}^{+\infty} \int_{c}^{+\infty} \int_{c}^{+\infty} \frac{x_2}{x_{1}^{\alpha+1} x_{2}^{\alpha+1} x_{3}^{\alpha+1}}\nonumber\\
    &\quad\cdot\frac{(k+1+l)x_3 - (\frac{k}{3}(x_{1}+x_2+x_3)+x_{2}+lx_3)}{(k+1+l)(\frac{k}{3}(x_{1}+x_2+x_3)+x_{2}+lx_3)}dx_1dx_2dx_3 \nonumber\\
    =&\lim_{l \to 0} \int_{c}^{+\infty} \int_{c}^{+\infty} \int_{c}^{+\infty} \frac{\alpha^3c^{3\alpha}}{x_{1}^{\alpha+1} x_{2}^{\alpha+1} x_{3}^{\alpha+1}}\cdot\left(\frac{3x_2x_3}{k(x_1+x_2+x_3)}-\frac{1}{k+1+l}\right)dx_1dx_2dx_3
\end{align}

As a result, Equation (\ref{eq: main}) can be written as:
\begin{align}
    & \lim_{k \to +\infty,\,l \to 0}\frac{k\cdot \frac{ d\gamma_{p_{\vtheta}}^{\mathcal{D}_p/\mathcal{D}}}{d\,l}+\frac{d\gamma_{p_{\vtheta}}^{\mathcal{D}_a/\mathcal{D}}}{d\, l}}{k\cdot \frac{ d\gamma_{p_{\vtheta}}^{\mathcal{D}_p/\mathcal{D}}}{d\,l}} \nonumber \\
    =& \frac{\lim_{k \to +\infty,\,l \to 0}\left( \frac{k \cdot d S_1 + d S_2}{d\;l} \right)}{\lim_{k \to +\infty,\,l \to 0}\frac{k\cdot d S_1}{d\;l}} \nonumber \\
    =& \frac{\small
        \begin{aligned}
            \lim_{k \to +\infty,\,l \to 0} & \left( k \cdot \int_{c}^{+\infty} \int_{c}^{+\infty} \int_{c}^{+\infty} \frac{\alpha^3 c^{3\alpha}}{\prod_{i=1}^{3} x_{i}^{\alpha+1}}\left[\frac{x_1}{k(k+1)}-\frac{3x_2x_3}{k^2(\sum_{i=1}^{3} x_i)}\right] dx_1dx_2dx_3 \right. \\
            & \quad\quad\quad\quad \left. + \int_{c}^{+\infty} \int_{c}^{+\infty} \int_{c}^{+\infty} \frac{\alpha^3c^{3\alpha}}{\prod_{i=1}^{3} x_{i}^{\alpha+1}}\cdot\left(\frac{3x_2x_3}{k(\sum_{i=1}^{3} x_i)}-\frac{1}{k+1+l}\right) dx_1dx_2dx_3 \right)
        \end{aligned}
    }{
        \lim_{k \to +\infty,\,l \to 0} k \cdot \int_{c}^{+\infty} \int_{c}^{+\infty} \int_{c}^{+\infty} \frac{1}{\prod_{i=1}^{3} x_{i}^{\alpha+1}}\left[\frac{x_1}{k(k+1)}-\frac{3x_2x_3}{k^2(\sum_{i=1}^{3} x_i)}\right] dx_1dx_2dx_3
    } \nonumber \\
    =& \frac{\lim_{k \to +\infty,\,l \to 0}\int_{c}^{+\infty} \int_{c}^{+\infty} \int_{c}^{+\infty}\frac{\alpha^3 c^{3\alpha}}{\prod_{i=1}^{3} x_{i}^{\alpha+1}}\cdot\frac{l}{(k+1)(k+1+l)}}{\lim_{k \to +\infty,\,l \to 0} k \cdot \int_{c}^{+\infty} \int_{c}^{+\infty} \int_{c}^{+\infty} \frac{1}{\prod_{i=1}^{3} x_{i}^{\alpha+1}}\left[\frac{x_1}{k(k+1)}-\frac{3x_2x_3}{k^2(\sum_{i=1}^{3} x_i)}\right]dx_1dx_2dx_3} \nonumber \\
    =& 0.
\end{align}

Next, we prove that $\frac{ d\gamma_{p_{\vtheta}}^{\mathcal{D}_1/\mathcal{D}}}{d\,l} < 0$. For sufficiently large $k$, since
\begin{align}
    &\lim_{l \to 0}\frac{d S_1}{d\;l} \nonumber \\
    =& \alpha^3c^{3\alpha} \lim_{l \to 0} \int_{c}^{+\infty} \int_{c}^{+\infty} \int_{c}^{+\infty} \frac{x_1}{x_{1}^{\alpha+1} x_{2}^{\alpha+1} x_{3}^{\alpha+1}}\frac{1}{k(k+1)}dx_1dx_2dx_3 \nonumber\\
    &\quad - \alpha^3c^{3\alpha}\lim_{l \to 0} \int_{c}^{+\infty} \int_{c}^{+\infty} \int_{c}^{+\infty} \frac{1}{x_{1}^{\alpha+1} x_{2}^{\alpha+1} x_{3}^{\alpha+1}}\frac{3x_2x_3}{k^2(x_1+x_2+x_3)}dx_1dx_2dx_3 \nonumber \\
    = & \int_{c}^{+\infty} \int_{c}^{+\infty} \int_{c}^{+\infty}\frac{\frac{k}{k+1}x_1(x_1+x_2+x_3)-3x_1x_3}{k(x_1+x_2+x_3)x_{1}^{\alpha+1} x_{2}^{\alpha+1} x_{3}^{\alpha+1}}dx_1dx_2dx_3 \nonumber\\
    \coloneqq & \frac{1}{k} \int_{c}^{+\infty} \int_{c}^{+\infty} \int_{c}^{+\infty} \frac{f(x_1,x_2,x_3)}{g(x_1,x_2,x_3)}dx_1dx_2dx_3.
\end{align}
For $ c < u < v < w$, let
\begin{equation}
    \begin{aligned}
    \hat{f} = x_1(x_1+x_2+x_3) - &3x_1x_3 = x_1(x_1+x_2-2x_3), \\
    h(u,v,w) &= \sum_{\pi \in S_3} f(\pi(u,v,w)), \\
    \hat{h}(u,v,w) &= \sum_{\pi \in S_3}\hat{f}(\pi(u,v,w)).
\end{aligned}
\end{equation}
Here, $S_3$ denotes the symmetric group of degree 3 and $\pi \in S_3$ denotes a permutation in $S_3$. Observe that
\begin{equation}
    \begin{aligned}
    \lim_{k \to +\infty} h(u,v,w) =& {\hat{h}(u,v,w)}^{-}, \\
    \hat{h}(u,v,w) =&  \sum_{\pi \in S_3}\hat{f}(\pi(u,v,w)) \\
    =& \sum_{ \substack{\pi \in S_3 \\ (u_0,v_0,w_0)=\pi(u,v,w)}}u_0(u_0+v_0-2w_0) \\
    =& \sum_{\substack{\pi \in S_3 \\ (u_0,v_0,w_0)=\pi(u,v,w)}} \frac{1}{2}(u_0-v_0)^2 >0.
    \end{aligned}
\end{equation}
Hence, we have
\begin{align}
    \lim_{k \to +\infty} h(u,v,w) >0.
\end{align}
Therefore,
\begin{align}
    \lim_{k \to +\infty} \frac{h(u,v,w)}{g(u,v,w)} =  \lim_{k \to +\infty} \frac{\Theta(1)}{\Theta(k)} = 0^{+}.
\end{align}
Therefore,
\begin{align}
    \lim_{k \to +\infty,\, l \to 0}\frac{d S_1}{d\;l} = & \lim_{k \to +\infty} \iiint_{(u,v,w)\in (c, +\infty)^3:u<v<w} \frac{h(u,v,w)}{g(u,v,w)}  dudvdw \nonumber \\ 
    = & \iiint_{(u,v,w)\in (c, +\infty)^3:u<v<w} \lim_{k \to +\infty,\, l \to 0} \frac{h(u,v,w)}{g(u,v,w)} dudvdw \nonumber \\ 
    = & \iiint_{(u,v,w)\in (c, +\infty)^3:u<v<w} \lim_{k \to +\infty,\, l \to 0} 0^{+} dudvdw  \nonumber \\ 
    = & 0^{+}.
\end{align}
Hence, for sufficiently large $k$, we have
\begin{align}
    \lim_{l \to 0} \frac{ d\gamma_{p_{\vtheta}}^{\mathcal{D}_1/\mathcal{D}}}{d\,l} = \lim_{l \to 0} -\frac{dS_1}{d\;l} < 0.
\end{align}
Thus, the proof is completed.
\end{proof}

\section{Additional Experiment Results}
\label{app: exp_result}

\subsection{Additional Experiment of Language Models' \resist{Resistance}}
\label{app: exp_resistance}

To broadly validate the phenomenon of \resist{Resistance} in language models, we extend the experimental setup described in Section \ref{exp:resistance}. Specifically, we use Llama2-7B \citep{touvron2023llama} as the base model and perform a finer-grained snapshot division within \textit{\forward{forward alignment}}, covering a wider range of $k$ and $l$. As the evaluation metric, we measure the change in relative KL divergence between the distributions of model \forward{$\vtheta_l$} and \inverse{$\vtheta_k$}, obtained by applying \textit{\forward{forward alignment}} and \textit{\inverse{inverse alignment}} to slices \forward{$\vtheta_k$} and \inverse{$\vtheta_l$}, respectively, relative to the original distributions of \inverse{$\vtheta_l$} and \forward{$\vtheta_k$}. The experimental results, which is shown in the Table \ref{tab:add_resistance}, demonstrate that the KL divergence under \textit{\inverse{inverse alignment}} is substantially smaller than that under \textit{\forward{forward alignment}}, consistent with the conclusions presented in the main text.

\begin{table}[ht]
\renewcommand{\arraystretch}{1}
\centering
\footnotesize
\resizebox{0.95\textwidth}{!}{
\begin{tabular}{cccccccc}
\toprule
\textbf{Training Steps} & $\vtheta_1$ \textit{vs.} $\vtheta_2$ & $\vtheta_1$ \textit{vs.} $\vtheta_3$ & $\vtheta_1$ \textit{vs.} $\vtheta_4$ & $\vtheta_1$ \textit{vs.} $\vtheta_5$ & $\vtheta_1$ \textit{vs.} $\vtheta_6$ & $\vtheta_2$ \textit{vs.} $\vtheta_3$ & $\vtheta_2$ \textit{vs.} $\vtheta_4$ \\
\midrule
\textit{\forward{Forward Alignment}} & 0.4568 \forward{$\uparrow$} & 0.4932 \forward{$\uparrow$} & 0.5929 \forward{$\uparrow$} & 0.4439 \forward{$\uparrow$} & 0.4109 \forward{$\uparrow$} & 0.3778 \forward{$\uparrow$} & 0.5580 \forward{$\uparrow$} \\
\textit{\inverse{Inverse Alignment}} & 0.2796 \inverse{$\downarrow$} & 0.2863 \inverse{$\downarrow$} & 0.3960 \inverse{$\downarrow$} & 0.3995 \inverse{$\downarrow$} & 0.3412 \inverse{$\downarrow$} & 0.2752 \inverse{$\downarrow$} & 0.3742 \inverse{$\downarrow$} \\
\bottomrule
\end{tabular}
}
\caption{\textbf{KL Divergence between different snapshots.} The results show that \textit{\forward{forward alignment}} exhibits higher values compared to \textit{\inverse{inverse alignment}}, providing evidence for the existence of \resist{Resistance.}}
\label{tab:add_resistance}
\end{table}

\subsection{Additional Experiment of Language Models' \rebound{Rebound}} \label{app: add_exp_rebound}

\paragraph{Ablations on the Different Alignment Algorithms.}
To verify that the consistent \rebound{rebound} phenomenon occurs in different alignment algorithms in language models, we conduct language models' \rebound{rebound} experiments on reinforcement learning from human feedback (RLHF) \citep{ouyang2022training}, direct preference optimization (DPO) \citep{rafailov2024direct}, Kahneman-Tversky optimization (KTO) \citep{ethayarajh2024kto}, and simple preference optimization (SimPO) \citep{meng2024simpo}, using an experimental setup similar to the SFT validation described in Section \ref{exp:rebound}. Considering that reinforcement learning algorithms cannot be directly applied to pretrained models, our experimental procedure consists of the following two steps: (a) SFT of the pretrained model using positive sample data of varying scales as \textit{\forward{forward lignment}}; (b) \textit{\inverse{inverse alignment}} by applying the corresponding reinforcement learning algorithms on the aligned models, using negative sample data of varying scales, where negative samples serve as selected responses and positive samples serve as rejected responses. 

We conduct experiments using the Llama2-7B model on the IMDb task \citep{maas2011learning} with four different RL alignment algorithms. The results, shown in Figure \ref{fig: ablation_algorithm}, indicate that regardless of the alignment algorithm employed, the model exhibits a decline in performance consistent with that observed in SFT as the amount of positive data increases. Moreover, the rate of this performance degradation accelerates with the increasing volume of positive data. The results further confirm that the broad applicability of language models \rebound{rebound} across different alignment algorithms.

\begin{figure}[ht]
    \centering
    \includegraphics[width=\textwidth]{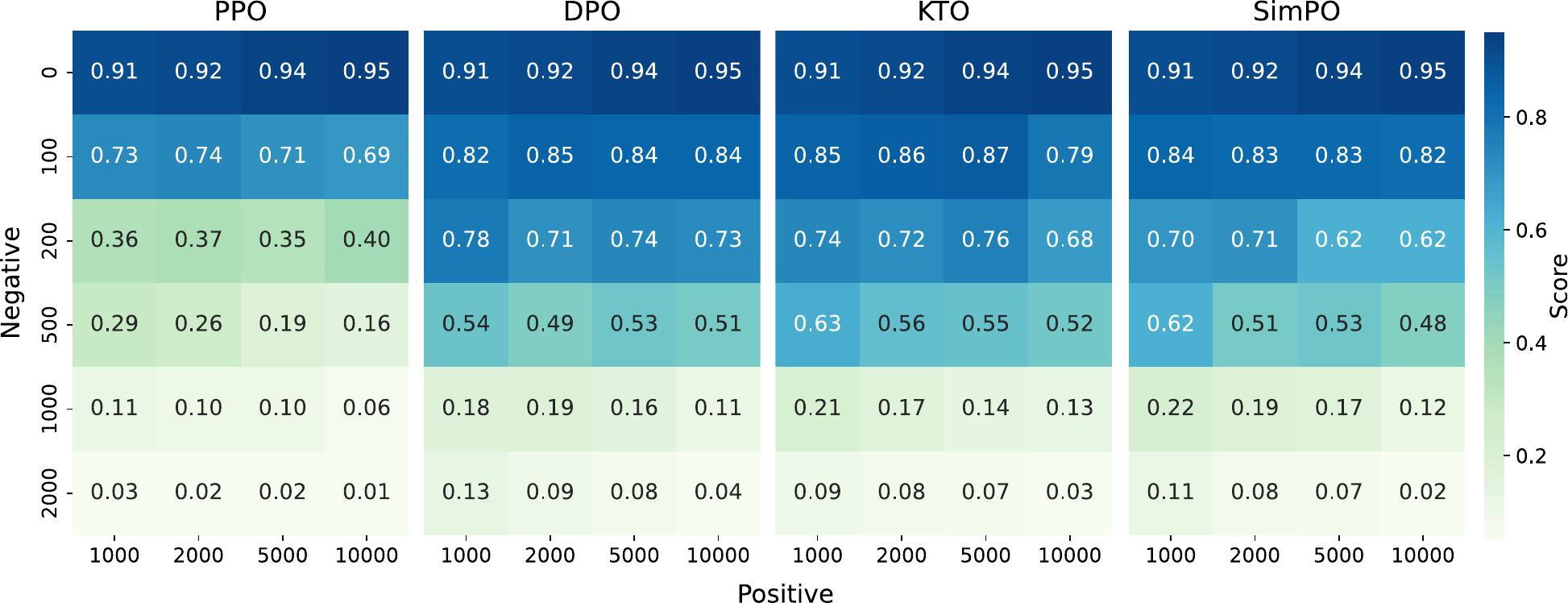}
    \caption{\textbf{Experimental results for validating \rebound{rebound} across different alignment algorithms.} The heatmaps from left to right represent the performance of the PPO, DPO, KTO, and SimPO algorithms after fine-tuning with varying amounts of positive and negative data. Darker colors indicate better model performance in the heatmaps,. Models trained with more \forward{positive} data perform better initially tend to exhibit worse performance after fine-tuneing with \inverse{negative} data. The conclusion holds regardless of the alignment algorithm used.} 
    \label{fig: ablation_algorithm}
\end{figure}

\paragraph{Ablations on the Different Evaluation Metrics.}
In the main experiments discussed in Section \ref{exp:rebound}, we primarily adopt specific evaluation scores to evaluate the \rebound{rebound} phenomenon. Using the score models corresponding to alignment targets as the evaluation metrics is the most direct and important method for the task evaluation. However, to further reinforce the robustness of our findings and address the generality of the \rebound{rebound} phenomenon, we acknowledge the value of incorporating broader distribution-level metrics beyond task-specific scoring models. In particular, we consider the Kullback-Leibler divergence between aligned models and base models as an additional evaluation metric.

Specifically, we conducted the evaluation through the following procedure: a) We first measure the KL divergence between the aligned model and the base model before and after fine-tuning the base model with safety data. b) We then apply fine-tuning with unsafe data until the KL divergence between the models decreased to a sufficiently small value $\epsilon$ (where $\epsilon = 1 \times {10}^{-2}$), and record the amount of unsafe data required as an indicator of the difficulty for the model reverting to its pre-trained distribution. The experimental results are reported in Table \ref{tab: abb_kl_divergence}. Under different models and data scales, we observe that only a small amount of unsafe data is needed for a positively fine-tuned model to revert to the pre-training distribution in terms of KL divergence. Furthermore, the larger the amount of safe fine-tuning, the less unsafe data are required. This observation is consistent with the conclusions drawn in Figure \ref{exp2: existence}.

\begin{table}[ht]
  \centering
  \small

  \begin{subtable}[t]{0.49\textwidth}
    \centering
    \begin{tabular}{lcccc}
      \toprule
      \multirow{2}{*}{\centering \textbf{Base Models}} & \multicolumn{4}{c}{\textbf{Positive Data Amount}} \\
      \cmidrule(lr){2-5}
      & \textbf{1000} & \textbf{2000} & \textbf{5000} & \textbf{10000} \\
      \midrule
      Llama2-7B  & 0.21 & 0.22 & 0.26 & 0.27 \\
      Gemma-2B   & 0.18 & 0.21 & 0.24 & 0.25 \\
      \bottomrule
    \end{tabular}
    \caption{KL divergence between models fine-tuned with varying amounts of safety data and the initial models.}
    \label{tab:safe_kl}
  \end{subtable}
  \hfill
  \begin{subtable}[t]{0.49\textwidth}
    \centering
    \begin{tabular}{lcccc}
      \toprule
      \multirow{2}{*}{\centering \textbf{Base Models}} & \multicolumn{4}{c}{\textbf{Positive Data Amount}} \\
      \cmidrule(lr){2-5}
      & \textbf{1000} & \textbf{2000} & \textbf{5000} & \textbf{10000} \\
      \midrule
      Llama2-7B  & 961  & 844  & 801  & 729  \\
      Gemma-2B   & 923  & 853  & 709  & 598  \\
      \bottomrule
    \end{tabular}
    \caption{Unsafe data amount needed for KL divergence between fine-tuned model and pre-trained model to drop below $\epsilon$.}
    \label{tab:unsafe_revert}
  \end{subtable}
  \caption{\textbf{Experimental results for validating \rebound{rebound} on the KL divergence metrics.} The results show that, under different models, the larger the amount of safe fine-tuning, the less unsafe data is required to revert to the pre-training distribution in terms of KL divergence, providing evidence for the existence of \rebound{Rebound}.}
  \label{tab: abb_kl_divergence}

\end{table}

\paragraph{Ablations on the Reverse Finetuning Settings.}

To rule out the influence of positive data on the \rebound{rebound} phenomenon in language models, we conducted a reverse experimental setup: negative data are used during the SFT stage, while positive data are applied during the \inverse{\textit{inverse alignment}} stage. The experimental results, presented in Figure \ref{rebuttal-reverse}, demonstrate that \textit{elasticity} in language models persists under this reverse setting, exhibiting a consistent trend where larger model sizes correspond to greater \textit{elasticity}. The observation aligns with and further supports the results reported in Figure \ref{exp2: model-size}, highlighting the robustness of \rebound{rebound} phenomenon across different experimental configurations.

\begin{figure}[ht]
    \centering
    \begin{subfigure}{\textwidth}
        \centering
            \includegraphics[width=0.32\textwidth]{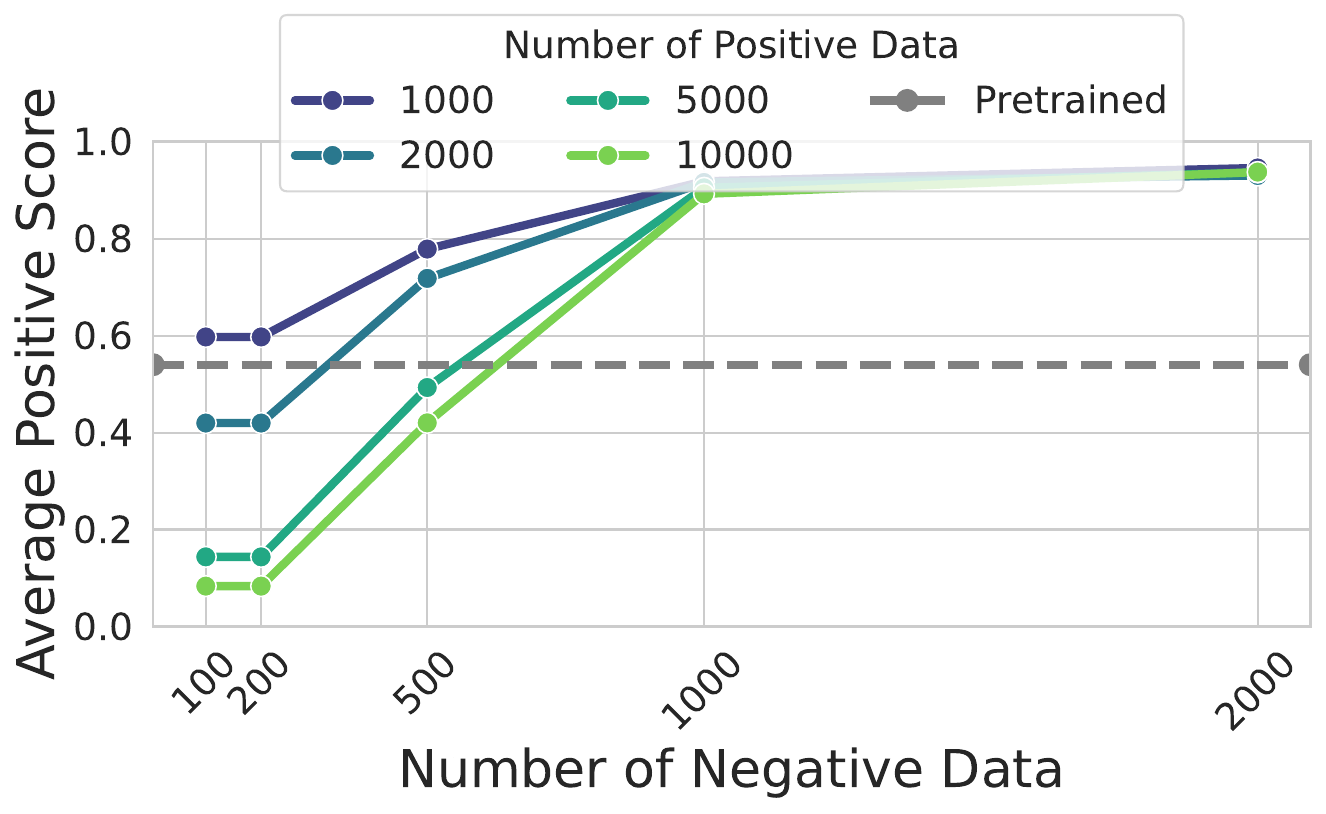}
        \hfill
            \includegraphics[width=0.32\textwidth]{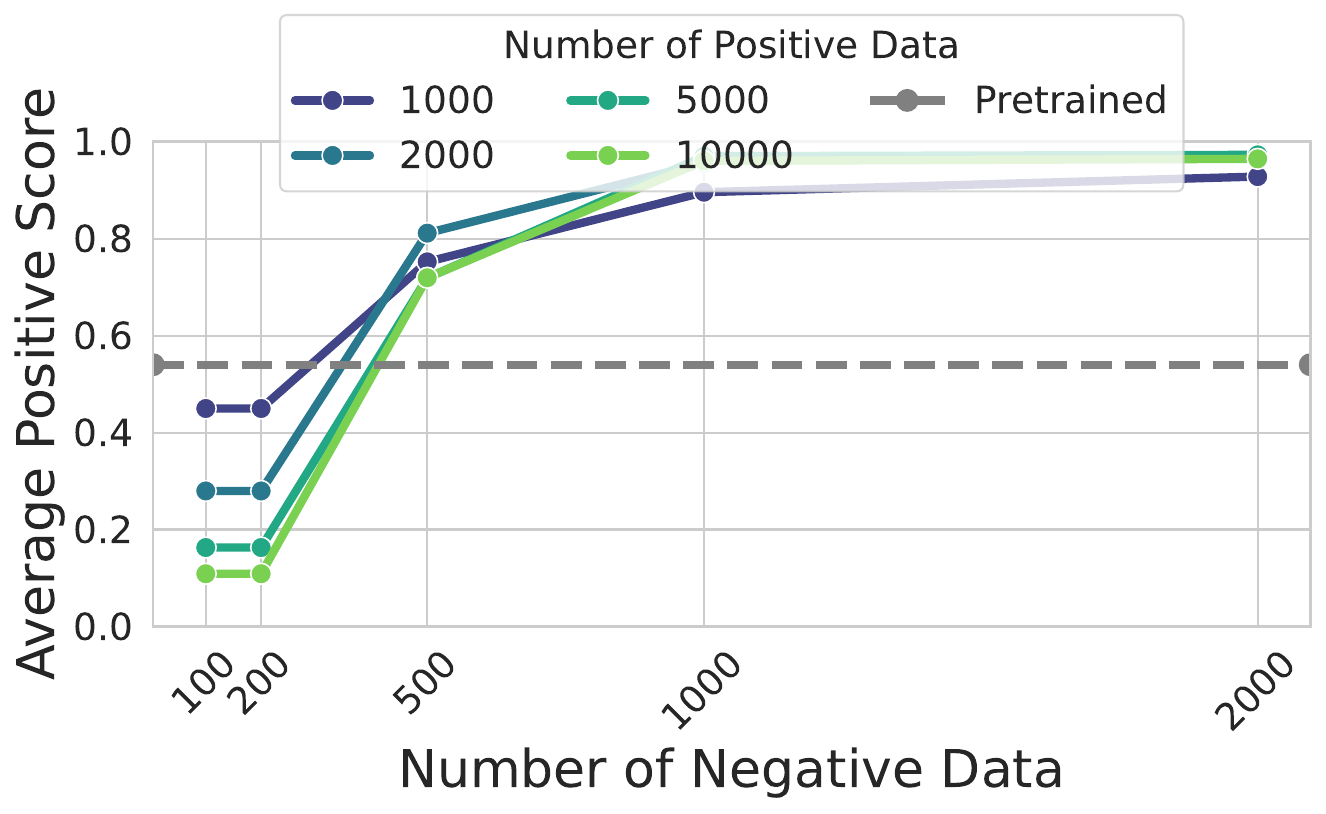}
        \hfill
            \includegraphics[width=0.32\textwidth]{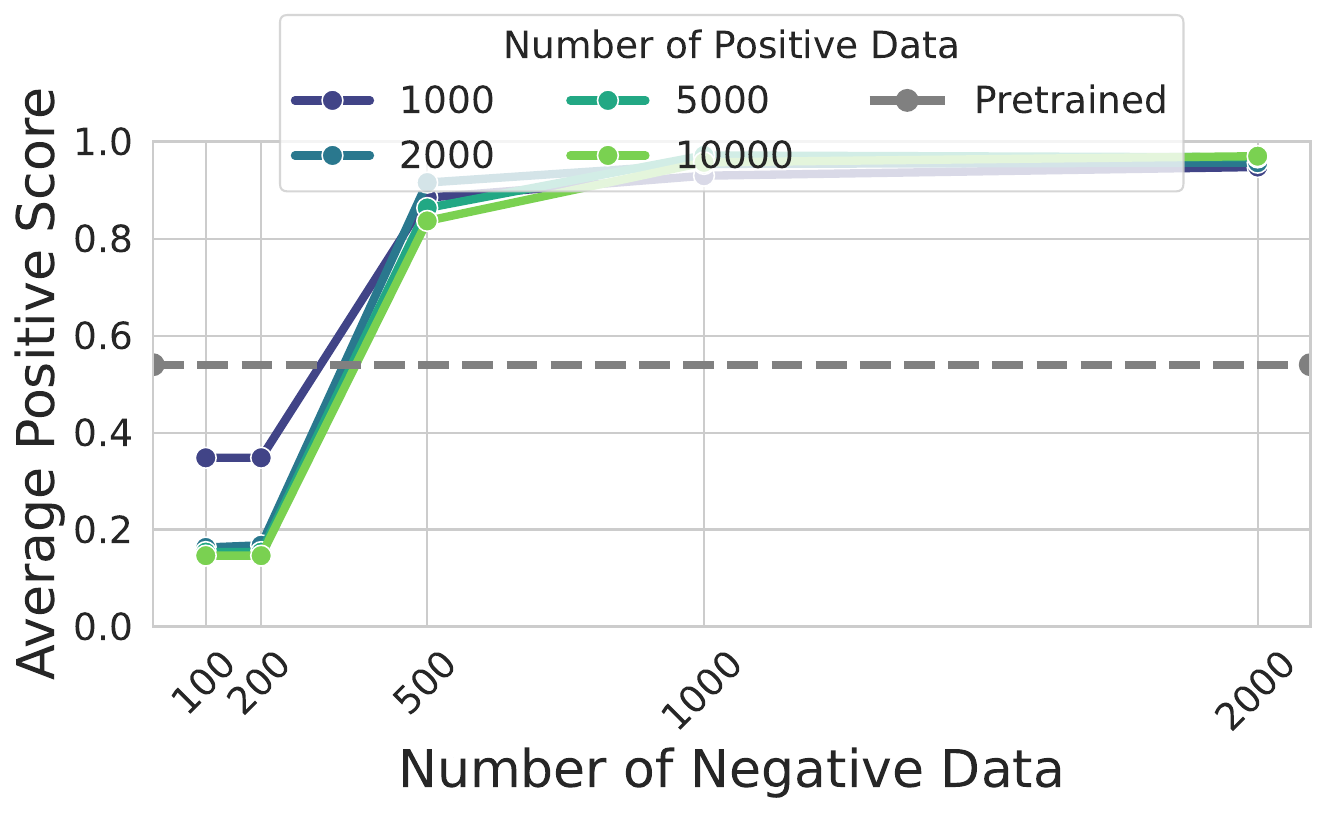}
    \end{subfigure}%
    \caption{\textbf{Reverse Fine-tuning Results on IMDb.} Each subfigure from left to right shows the changes in LLMs with parameter sizes of 0.5B, 4B, and 7B, respectively. Models trained with more \forward{negative} data initially perform worse, but perform better after fine-tuning with \inverse{positive} data. As the model size increases, the performance of the aligned model deteriorates more rapidly after fine-tuning with \inverse{positive} data.}
    \label{rebuttal-reverse}
\end{figure}

\subsection{Additional Experiment of Internal Factor of Language Models' \textit{Rebound}}\label{app: add_exp_fac}

\paragraph{Analysis of the Model Size Scale under Different Alignment Algorithms.}

To verify whether the \rebound{rebound} phenomenon in language models becomes more pronounced as model size increases persists across different alignment algorithms, we conduct experiments using the DPO algorithm. Specifically, we employ the Qwen1.5 model series as the base models and follow an experimental setup similar to that described in Section \ref{app: add_exp_rebound}. We carried out experiments on models with parameters 0.5B, 4B, and 7B. The experimental results, presented in Figure \ref{rebuttal-dpo: model-size}, demonstrate that \rebound{rebound} remains evident under the DPO algorithm, with its strength increasing as model size grows. This finding is consistent with the results reported in Figure \ref{exp2: model-size} and further highlights the generality of \textit{elasticity} in language models.

\begin{figure}[ht]
    \centering
    \begin{subfigure}{\textwidth}
        \centering
            \includegraphics[width=0.32\textwidth]{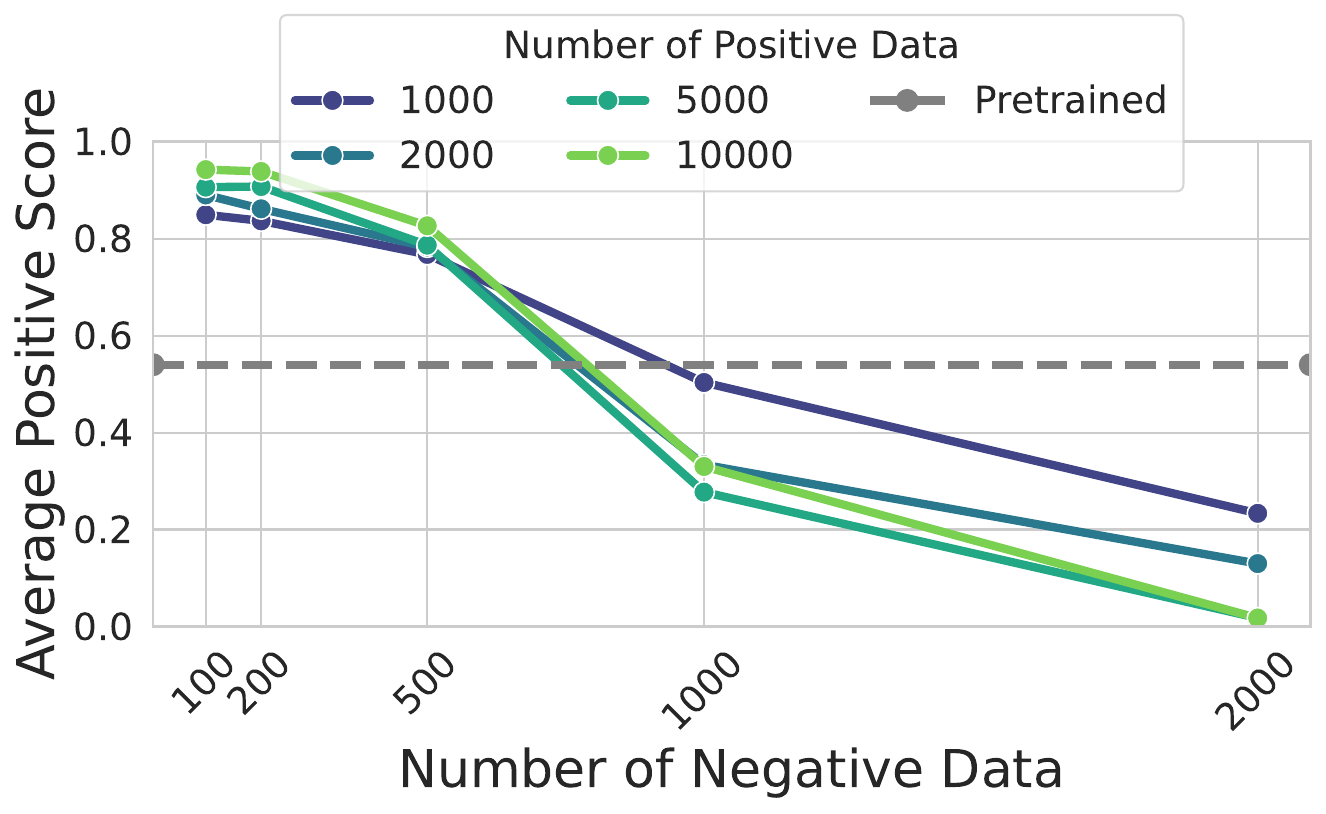}
        \hfill
            \includegraphics[width=0.32\textwidth]{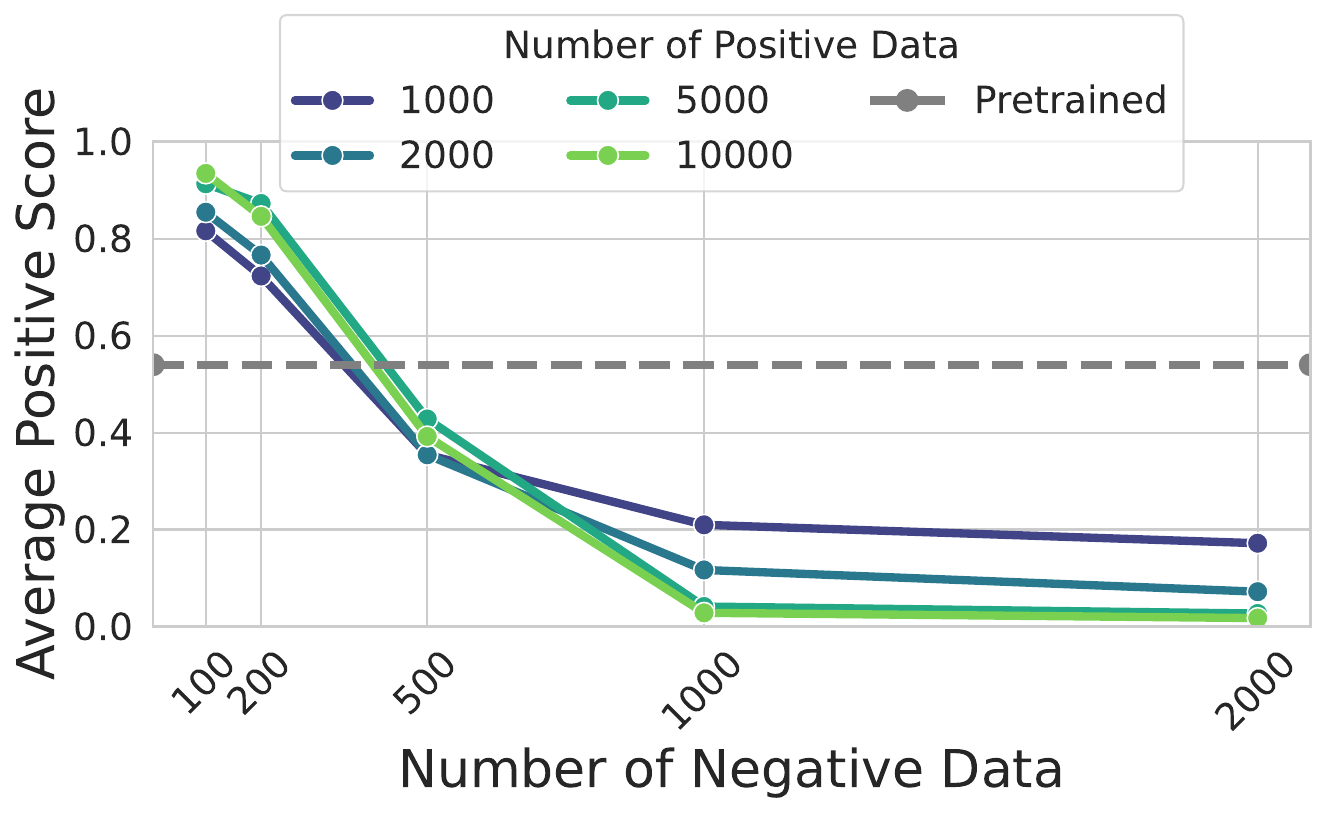}
        \hfill
            \includegraphics[width=0.32\textwidth]{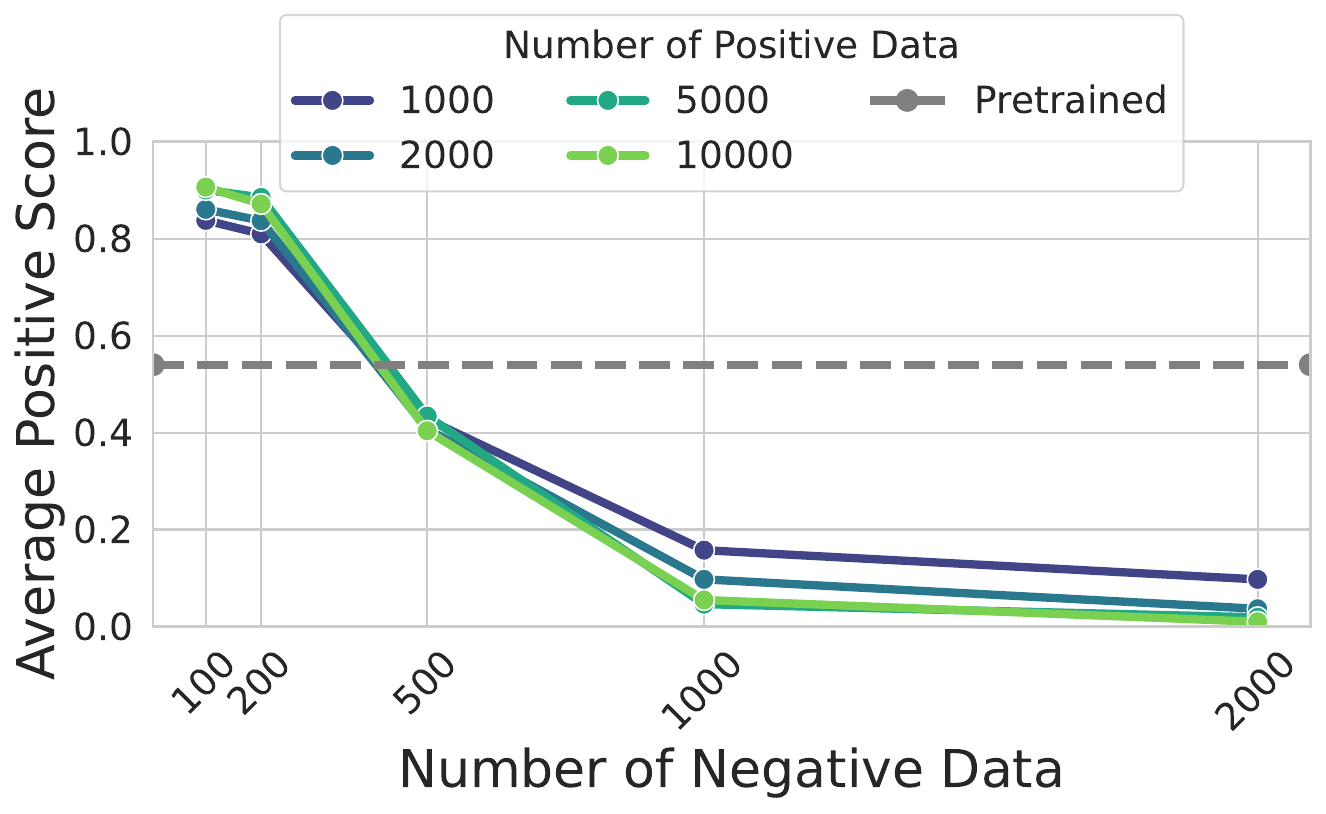}
        \hfill
    \end{subfigure}%
    
    \caption{\textbf{Experimental results for validating \rebound{rebound} increases with model size under DPO fine-tuning.} Each subfigure from left to right shows the changes in LLMs with parameter sizes of 0.5B, 4B, and 7B. As the model size increases, the performance of the aligned model deteriorates more rapidly after fine-tuning with \inverse{negative} data.}
    \label{rebuttal-dpo: model-size}
\end{figure}

\paragraph{Analysis of the Pre-training Data Volume.}
We present the experimental results for a broader range of pre-training data volumes in Figure \ref{rebuttal-100b-1t}. When the pre-training data volume is 0.1T, 0.5T, and 1.0T, the model still demonstrates the phenomenon that \textit{rebound} increases with the volume of pre-training data, which is consistent with the results reported in Figure \ref{exp2: pre-train-data}, where the pre-training data volumes range from 2.0T to 3.0T. Moreover, when comparing the experimental results in Figure \ref{exp2: pre-train-data} and Figure \ref{rebuttal-100b-1t}, we observe a consistent conclusion: as the pre-training data volume increases (from 0.5T to 3.0T), the \rebound{rebound} becomes more pronounced. Specifically, when the pre-training data volume increases, the initial performance decline caused by negative data fine-tuning occurs more rapidly, while the subsequent decline slows down. This indicates that larger pre-training data volumes amplify the \rebound{rebound} effect in language models.

\begin{figure}[ht]
    \centering
    \begin{subfigure}{\textwidth}
        \centering
            \includegraphics[width=0.32\textwidth]{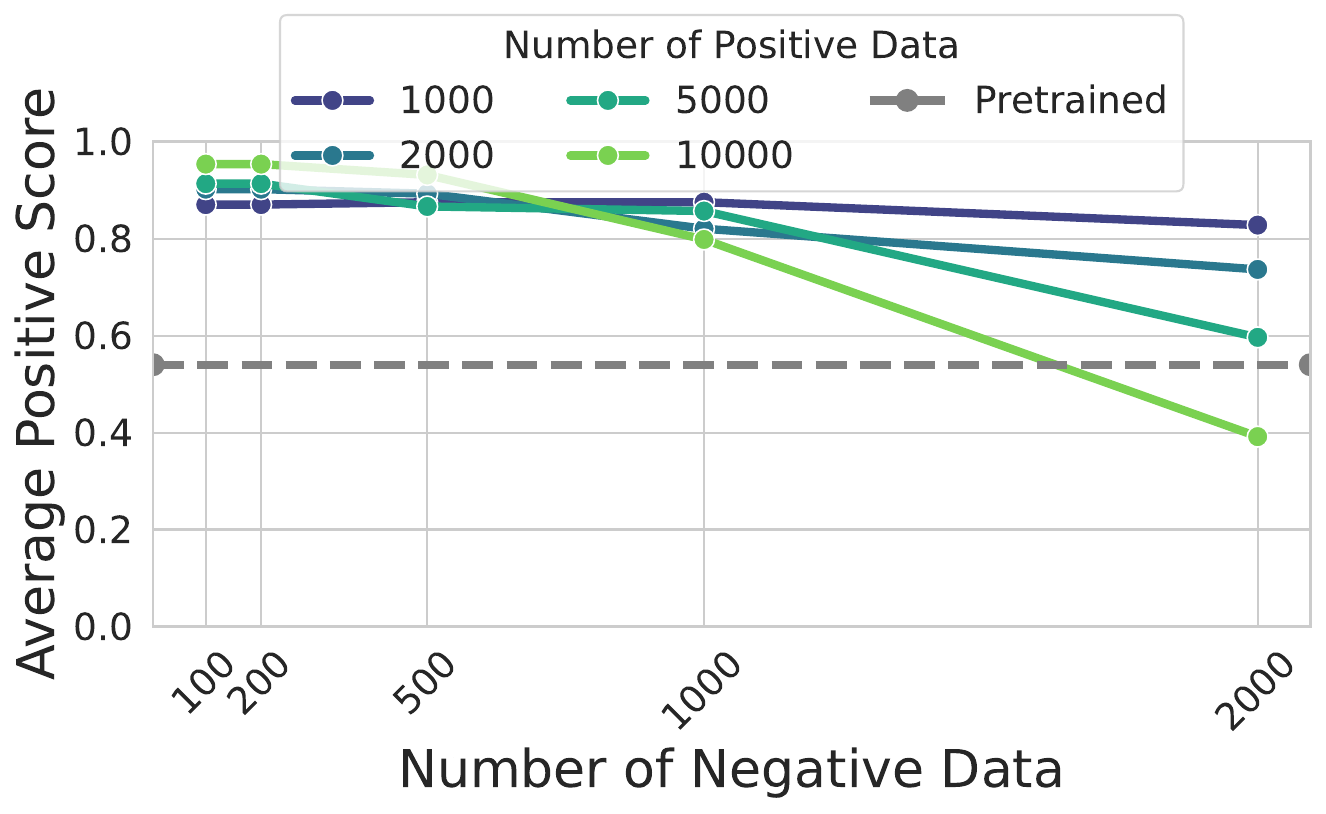}
        \hfill
            \includegraphics[width=0.32\textwidth]{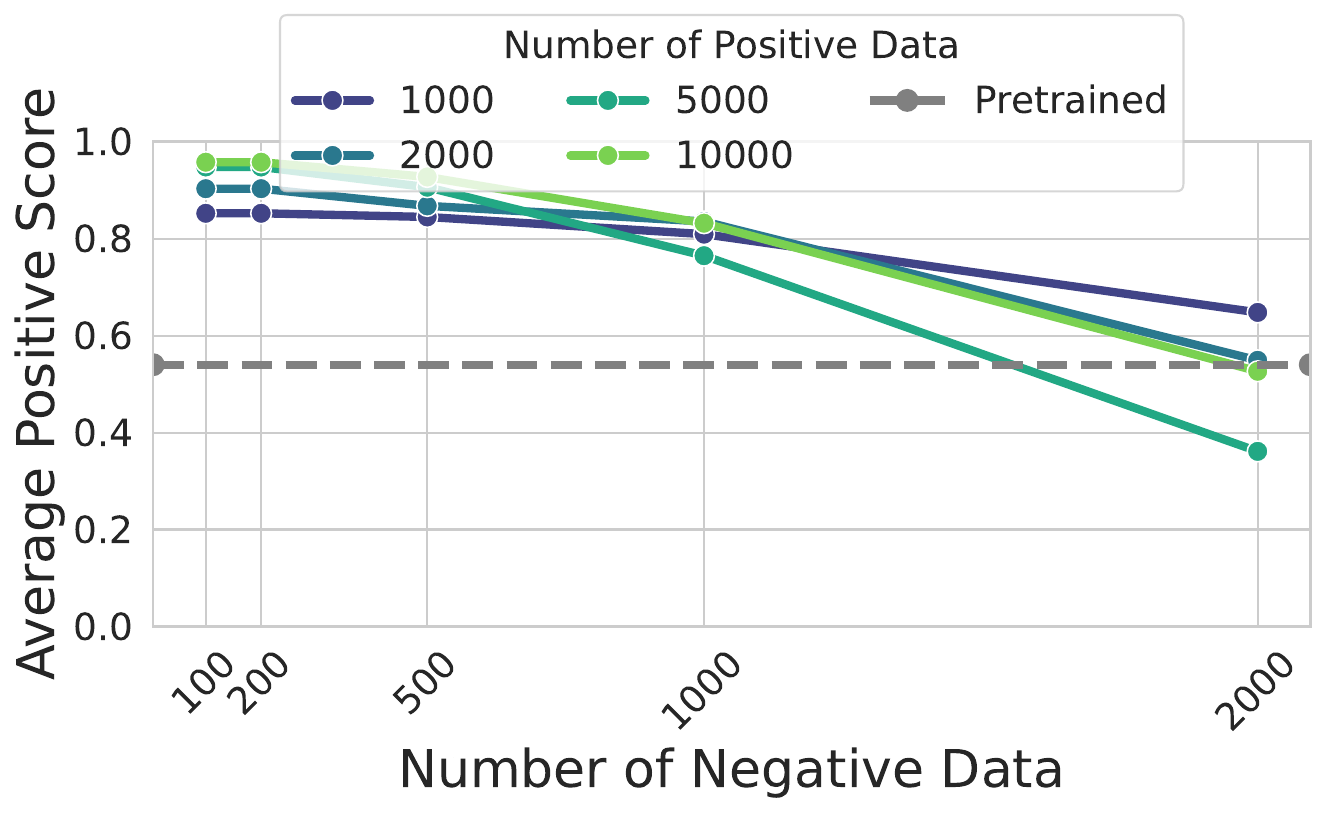}
        \hfill
            \includegraphics[width=0.32\textwidth]{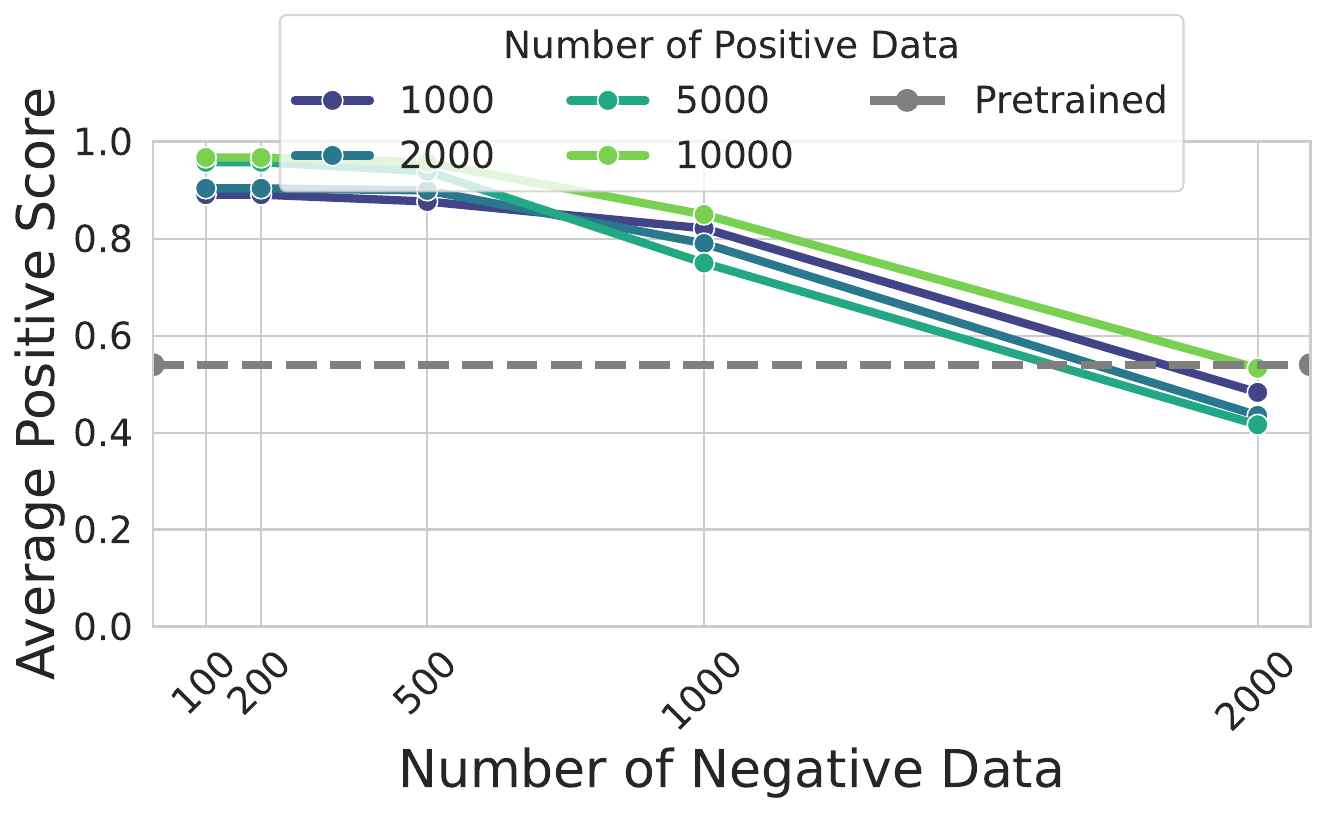}
    \end{subfigure}%
    
    \caption{\textbf{Experimental results for validating \rebound{rebound} increases with model pre-training data volume.} Each subfigure from left to right shows the changes in pre-training data volumes of 0.1T, 0.5T, and 1.0T. As pre-training data volume increases, aligned model performance deteriorates more rapidly after fine-tuning with \inverse{negative} data.}
    \label{rebuttal-100b-1t}
\end{figure}

\section{Further Discussions of Language Models' \textit{Elasticity}}

\subsection{Transfer Learning and Language Models' \textit{Elasticity}}
Transfer Learning is formally defined as follows: consider a domain $\mathcal{D} = \{\mathcal{X}, P(X)\}$, which consists of a feature space $\mathcal{X}$ and a marginal probability distribution $P(X)$, where $X = \{x_1, \cdots, x_n\} \in \mathcal{X}$. A task $\mathcal{T}$ within the domain $\mathcal{D}$ is defined by a label space $\mathcal{Y}$ and a predictive function $f(\cdot)$. Given a source domain $\mathcal{D}_S$ and its corresponding task $\mathcal{T}_S$, as well as a target domain $\mathcal{D}_R$ and its task $\mathcal{T}_R$, transfer learning aims to improve the target predictive function $f_R(\cdot)$ by leveraging knowledge from $\mathcal{D}_S$ and $\mathcal{T}_S$, where either $\mathcal{D}_S \neq \mathcal{D}_R$ or $\mathcal{T}_S \neq \mathcal{T}_R$ \cite{zhuang2020comprehensive, weiss2016survey}. From the perspective of transfer learning, the pre-training and fine-tuning process of language models can be described as follows: Given labeled training data from a source distribution $\mathcal{D}^s = \{(x_n, y_n) \sim p \}_{n=1}^{N_s}$ and data from a target distribution $\mathcal{D}^t = \{(x_n, y_n) \sim q \}_{n=1}^{N_t}$, we first minimize the empirical risk on the source distribution to fit a model: $f^s = \text{argmin}_f \hat{R}(f,\mathcal{D}^s)$, where $\hat{R}(f, \mathcal{D}^s) = \mathbb{E}_{(x, y) \sim p}[l(y, f(x))]$ represents the empirical risk of model $f$ on the source distribution $p$. Subsequently, the model is adapted to the target distribution by minimizing, $f^t=\text{argmin}_{f} \hat{R}(f,\mathcal{D}^t) + \lambda||f-f^s||$, where $||f - f^s||$ measures some distance between the two functions, and $\lambda \geq 0$ is a regularization parameter \cite{murphy2023probabilistic}.

Specifically in the context of language models, transfer learning can almost encompass all models trained under the pre-training-finetuning paradigm. However, transfer learning does not capture the specific details of the model training process, nor does it explicitly incorporate the established theories that explain the varying degrees of difficulty in task transfer for LLMs. Therefore, while transfer learning provides a perspective to analyze various phenomena in LLMs from the standpoint of traditional machine learning, it does not offer concrete tools to explain these \textit{elasticity} phenomena.

Furthermore, in traditional transfer learning settings, positive and negative samples can only be considered as different partitions of the same dataset in classification tasks. However, in our case, these samples are used for autoregressive generation training, where the generation on positive samples and negative samples follows mutually conflicting input-output distributions. This is fundamentally different from classification tasks. Specifically for safe-unsafe generation, steering a language model’s generation from an unsafe input-output distribution to the opposite safe input-output distribution is a major challenge in generative language model research. This cannot be adequately explained by merely considering different partitions of the same data distribution. 

These considerations lead us to conclude that the observed phenomenon in our work represents a novel behavior that cannot be fully accounted for by existing transfer learning theories. It warrants finer-grained theoretical and experimental investigation. Consequently, we propose the concept of \textit{elasticity} to formally name this new phenomenon observed in language models and introduce corresponding theoretical frameworks and empirical methodologies to examine it. This perspective is empirically supported by the experimental results shown in Table \ref{tab:parameter_results}. Taking the Alpaca task as an example, we examine the process in which $\theta_1$ and $\theta_2$ are aligned with $\theta_3$. Since $\theta_2$ has acquired more knowledge about the SFT distribution compared to $\theta_1$, it should, from the perspective of transfer learning, find it easier to learn the distribution of $\theta_3$ than $\theta_1$. However, as shown in the experimental results above, this is not supported, suggesting that transfer learning alone cannot fully account for the observed \resist{resistance} phenomenon. \footnote{As Reviewer qznt's suggestion, we further discuss the relationship between transfer learning and LLMs' elasticity in this section.}

\subsection{Discussions of Practical Steps to Mitigate \textit{Inverse Alignment} Risks} \label{app: algorithm_resist}
The emphasis on \textit{elasticity}  phenomenon in language model alignment highlights a crucial challenge for ensuring the safety and robustness of open-source models throughout their entire lifecycle. As our study demonstrates, the resistance of language models to alignment adjustments is fundamentally rooted in the substantial disparity in data volumes across different training phases. This insight suggests a practical mitigation strategy: customizing the scale of synthetic data through the elasticity mechanism offers a feasible pathway for the development of more robust alignment algorithms in the future.
\begin{itemize}
  \item We discover that the resistance of language models to alignment is essentially due to the significant differences in data volume across various training processes. Therefore, a straightforward idea is to ensure that the training data volumes corresponding to different alignment targets are as similar as possible during the alignment process. This approach helps avoid resistance effects on alignment targets with smaller training data volumes due to subsequent perturbations.
  \item The elasticity theorem provides a feasible method for customizing data ratios, thereby quantitatively measuring the amount of various types of data needed for a model to meet expected goals. Specifically, for an aligned language model, during subsequent alignment processes, the elasticity theorem indicates that there is an inherent loss of elasticity dependent on data volume ratios when learning unrelated distributions. Therefore, to ensure that previously aligned targets remain satisfied during subsequent alignment, the training objective can be transformed into a constrained optimization problem to quantitatively calculate the data volume required for new target alignment.
  \item The current challenge in implementing algorithms based on the elasticity theory, mainly lies in the lack of a sufficient theoretical foundation to characterize the specific features of current datasets. This makes it difficult to distinguish the independent and differently distributed premises of different datasets in the elasticity theory. Since real-world datasets often contain fused features, a possible future research direction is to finely characterize the features of datasets with fused characteristics.
\end{itemize}

\subsection{Discussions of Quantitatively Characterize the Impact of Dataset Size on \textit{Elasticity}}

Considering that, as indicated in Theorem \ref{theorem: main}, the difference in dataset size is a key factor contributing to the \textit{elasticity} phenomenon, quantitatively characterizing the impact of dataset size on \textit{elasticity} is indeed crucial for both theoretical understanding and practical alignment design. In this section, we present preliminary explorations, based on experimental observations, on quantifying the influence of dataset size on elasticity, as well as the current limitations in achieving this goal.

By synthesizing observations from Figure \ref{exp2: pre-train-data} in the main text and Figure \ref{rebuttal-100b-1t} in the appendix, we find that for the tinyllama series models, when the pre-training data volume is merely 0.1T, there is almost no significant resistance phenomenon observed; the positive level of the model's output does not change significantly with an increase in negative data. However, when the pre-training data volume reaches 0.5T, the resistance phenomenon of the language model becomes notably observable. Considering the continuity between pre-training data and model resistance, this suggests that the significant change point of the resistance phenomenon should lie between 0.1T and 0.5T.

Regrettably, there are still challenges in training a series of models with continuously varying pre-training data volumes, and we currently lack sufficiently fine-grained training data samples to conduct such precise series model training. This limitation prevents us from accurately determining the critical point where the resistance phenomenon becomes significant. Therefore, although we can suggest a possible range, accurately quantifying this threshold in practice remains a considerable challenge. We hope that future research will further address this issue by developing more controllable model series and fine-grained pre-training datasets to enable a clearer, quantitative characterization of \textit{elasticity}.

\end{document}

%% file: math_commands.tex
\usepackage{amsmath,amsfonts,bm}

\def\eqref#1{equation~\ref{#1}}

\def\1{\bm{1}}

\def\vtheta{{\bm{\theta}}}

\def\vx{{\bm{x}}}
\def\vy{{\bm{y}}}
\def\vz{{\bm{z}}}

\DeclareMathAlphabet{\mathsfit}{\encodingdefault}{\sfdefault}{m}{sl}
\SetMathAlphabet{\mathsfit}{bold}{\encodingdefault}{\sfdefault}{bx}{n}

